\documentclass{article} 
\usepackage{nips14submit_e,times}
\usepackage{hyperref}
\usepackage{url}

\usepackage{amsmath}
\usepackage{amsfonts}
\usepackage{amssymb,tikz}
\usepackage{graphicx}
\usepackage{hyperref} 
\usepackage{bbm}
\usepackage{algorithm,algorithmic}
\usepackage{float}	
\usepackage{booktabs}
\usepackage{subfig,color}
\usepackage[sort,nonamebreak, sectionbib,square]{natbib}
\usepackage[nottoc, notlof, notlot]{tocbibind}
\newenvironment{proof}{\paragraph{Proof:}}{\hfill$\square$}
\newtheorem{thm}{Theorem}[section]
\newtheorem{prop}[thm]{Proposition}
\newtheorem{lemma}[thm]{Lemma}

\newtheorem{dfn}{Definition}[section]

\newcommand{\OO}{\mathcal{O}}
\newcommand{\data}{\mathcal{D}}
\newcommand{\R}{\mathbb{R}}
\newcommand{\E}{\mathbb{E}}
\renewcommand{\P}{\mathbb{P}}
\newcommand{\ind}{\mathbbm{1}}
\newcommand{\eqdef}{:=}
\newcommand{\dd}{\mathrm{d}}
\newcommand{\Unif}{\mathcal{U}([0,1])}
\newcommand{\N}{\mathcal{N}}
\newcommand{\B}{\mathcal{B}}

\newcommand{\prior}{\pi_{\xi}(\theta)} 
\newcommand{\postarg}[2]{\pi_{\xi,#1}(#2|\data)}
\newcommand{\post}{\postarg{\gamma}{\theta}} 
\newcommand{\postt}{\postarg{\gamma_t}{\theta}} 
\newcommand{\posttm}{\postarg{\gamma_{t-1}}{\theta}} 
\newcommand{\ev}{Z_{\xi,\gamma}(\data)} 
\newcommand{\np}{n_{+}}
\newcommand{\nm}{n_{-}}

\newcommand{\postsl}{\postarg{\gamma}{\theta,z}}

\newcommand{\postGP}{\pi_{\xi,\gamma}(\dd s|\data)}
\newcommand{\priorGP}{\pi_{\xi}(\dd s)}
\newcommand{\datap}{\data_{+}}
\newcommand{\datam}{\data_{-}}
\newcommand{\postGPm}{\pi_{\xi,\gamma}(s_{1:n}|\data)}
\newcommand{\priorGPm}{\N_d\left(s_{1:n};0,K_\xi\right)}

\title{PAC-Bayesian AUC classification and scoring}

\author{
James Ridgway\thanks{\url{http://www.crest.fr/pagesperso.php?user=3328}} \\
CREST and CEREMADE University Dauphine\\
\texttt{james.ridgway@ensae.fr} 
\And
Pierre Alquier \\
CREST (ENSAE) \\
\texttt{pierre.alquier@ucd.ie} 
\And
Nicolas Chopin \\
CREST (ENSAE) and HEC Paris \\
\texttt{nicolas.chopin@ensae.fr} 
\And
Feng Liang \\
University of Illinois at Urbana-Champaign\\
\texttt{liangf@illinois.edu} 
}

\nipsfinalcopy 

\begin{document}

\maketitle

\begin{abstract}
We develop a scoring and classification procedure based on the PAC-Bayesian
approach and the AUC (Area Under Curve) criterion. We focus initially on the class of 
linear score functions. We derive PAC-Bayesian non-asymptotic bounds for two types of 
prior for the score parameters: a Gaussian prior, and a spike-and-slab
prior; the latter makes it possible to perform feature selection. 
One important advantage of our approach is that it is amenable to powerful
Bayesian computational tools. We derive in particular a Sequential Monte Carlo algorithm, as an efficient method which may be used as a gold standard, 
and an Expectation-Propagation algorithm, as a much faster but approximate method. 
We also extend our method to a class of non-linear score functions,
essentially leading to a nonparametric procedure, by considering a Gaussian process 
prior. 
\end{abstract}

\section{Introduction}

Bipartite ranking (scoring) amounts to rank (score) data from binary labels. 
An important problem in its own right, bipartite ranking is also an elegant way to 
formalise classification: once a score function has been estimated from the data, 
classification reduces to chooses a particular threshold, which determine to which
class is assigned each data-point, according to whether its score is above or below that
threshold. It is convenient to choose that threshold only once the score has
been estimated, so as to get finer control of the false negative and false positive rates; 
this is easily achieved by plotting the ROC (Receiver operating characteristic) curve. 

A standard optimality criterion for scoring is AUC (Area Under Curve), which measures
the area under the ROC curve. AUC is appealing for at least two reasons. First,
maximising AUC is equivalent to minimising the $L_1$ distance between the estimated
score and the optimal score. Second, under mild conditions, \cite{Cortes2003} show that 
AUC for a score $s$ equals the probability that $s(X^-)<s(X^+)$ for $X^-$ (resp. $X^+$) 
a random draw from the negative (resp. positive class). \cite{Yanetal2003} observed
AUC-based classification handles much better skewed classes (say the positive class is 
much larger than the other) than standard classifiers, because it enforces
a small score for all members of the negative class (again assuming the negative class
is the smaller one). 

One practical issue with AUC maximisation is that the empirical version of AUC is not a continuous function. One way to address this problem is to "convexify" this function, 
and study the properties of so-obtained estimators \citep{Clemencon2008a}. 
We follow instead the PAC-Bayesian approach in this paper, which consists of
using a random estimator sampled from a pseudo-posterior distribution that
penalises exponentially the (in our case) AUC risk. It is well known
\cite[see e.g. the monograph of][]{Catoni2007} that the PAC-Bayesian approach
comes with a set of powerful technical tools to establish non-asymptotic 
bounds; the first part of the paper derive such bounds. 
A second advantage however of this approach, as we show in the second part of the
paper, is that it is amenable to powerful Bayesian computational tools,
such as Sequential Monte Carlo and Expectation Propagation. 

\section{Theoretical bounds from the PAC-Bayesian Approach}\label{sec:theory}

\subsection{Notations}
The data $\data$ consist in the realisation of $n$ IID (independent and identically
distributed) pairs $(X_i,Y_i)$ with distribution $P$, and taking values in $\R^d\times\{-1,1\}$. Let 
$\np = \sum_{i=1}^n \ind\{Y_i=+1\}$, $\nm=n-n_+$. 
For a score function $s:\R^d \rightarrow \R$, the
AUC risk and its empirical counter-part may be defined as: 
\begin{align*}
 R(s) & = \P_{(X,Y),(X',Y')\sim P}\left[
\{s(X)-s(X')\}(Y-Y')<0
 \right], \\ 
 R_n (s) & = \frac{1}{n(n-1)} \sum_{i\neq j} \ind\left[ \{s(X_i)-s(X_j)\}
 (Y_i-Y_j)<0\right].
\end{align*}
Let $\sigma(x)=\E(Y|X=x)$, $\bar{R}=R(\sigma)$ and $\bar{R}_n=R_n(\sigma)$. 
It is well known that $\sigma$ is the score that minimise $R(s)$, i.e.
$R(s)\geq \bar{R}=R(\sigma)$ for any score $s$.

The results of this section apply to the class of linear scores, $s_\theta(x)=\left<\theta,x\right>$, where 
$\left<\theta,x\right>= \theta^T x$ denotes the inner product. Abusing notations, 
let $R(\theta)=R(s_\theta)$, $R_n(\theta)=R_n(s_\theta)$, and,
for a given prior density $\pi_\xi(\theta)$ that may depend on some hyperparameter
$\xi\in\Xi$, 
define the Gibbs posterior density (or pseudo-posterior) as 
$$ \post \eqdef 
\frac{\pi_\xi(\theta)\exp\left\{-\gamma R_n(\theta)\right\}}
{\ev},
\quad
\ev = \int_{R^d}\pi_\xi(\tilde{\theta})\exp\left\{-\gamma R_n(\tilde{\theta})\right\}\,\dd\tilde{\theta}
$$
for $\gamma>0$. Both the prior and posterior densities are defined with
respect to the Lebesgue measure over $\R^d$.

\subsection{Assumptions and general results}

Our general results require the following assumptions. 

\begin{dfn}
 We say that Assumption {\bf Dens}$(c)$ is satisfied for $c>0$ if
 $$ \mathbb{P}( \left<X_1 - X_2,\theta\right> \geq 0,
  \left<X_1-X_2,\theta'\right> \leq 0) \leq c \|\theta - \theta'\| $$
 for any $\theta$ and $\theta'\in\R^d$ such that $\|\theta\|=\|\theta'\|=1$.
\end{dfn}

This is a mild Assumption, which holds for instance
as soon as $(X_1-X_2)/\|X_1-X_2\|$ 
admits a bounded probability density; see the appendix. 

\begin{dfn}[Mammen \& Tsybakov margin assumption]
 We say that Assumption {\bf MA}$(\kappa,C)$ is satisfied for $\kappa\in[1,+\infty]$
 and $C\geq 1$ if
 $$
 \mathbb{E}\left[(q^{\theta}_{1,2})^2\right] \leq C \left[
 R(\theta)-\overline{R}\right]^{\frac{1}{\kappa}}
 $$
where $q^{\theta}_{i,j}=\ind\{\left<\theta,X_i - X_j\right>
 (Y_i-Y_j)<0\} - \ind\{[\sigma(X_i)-\sigma(X_j)]
 (Y_i-Y_j)<0\} - R(\theta)+\overline{R}$.
\end{dfn}

This assumption was introduced for classification by \cite{Mammen1999}, and
used for ranking by \cite{Clemencon2008} and \cite{Robbiano2013} (see also a
nice discussion in~\cite{Lecue}). 
The larger $\kappa$, the less restrictive {\bf MA}$(\kappa,C)$. In fact, 
{\bf MA}$(\infty,C)$ is always satisfied for $C=4$. 
For a noiseless classification task (i.e. $\sigma(X_i) Y_i \geq 0$ almost surely),
$\overline{R}=0$, 
\begin{align*}
  \mathbb{E}((q^{\theta}_{1,2})^2)  = {\rm Var}(q^{\theta}_{1,2}) 
  &  =  \mathbb{E}[\ind\{\left<\theta,X_1 - X_2\right>
 (Y_i-Y_j)<0\}] = R(\theta)-\overline{R}
\end{align*}
and {\bf MA}$(1,1)$ holds.
More generally, {\bf MA}$(1,C)$ is satisfied as soon as the noise is small; see the discussion in
Robiano 2013 (Proposition 5 p. 1256) for a formal statement.
From now, we focus on either {\bf MA}$(1,C)$
or {\bf MA}$(\infty,C)$, $C\geq 1$. 
It is possible
to prove convergence under {\bf MA}$(\kappa,1)$ for a general $\kappa\geq 1$,
but at the price of complications regarding the
choice of $\gamma$; see \cite{Catoni2007}, \cite{Alquier2008} and \cite{Robbiano2013}.

We use the classical PAC-Bayesian methodology initiated by
\cite{Shawe-Taylor1997, McAllester1998} (see
\cite{Catoni2007,Alquier2008} for a complete survey and more recent advances)
to get the following results. Proof of these and forthcoming
results may be found in the appendix. Let $\mathcal{K}(\rho,\pi)$ denotes the Kullback-Liebler divergence,
$\mathcal{K}(\rho,\pi)=\int \rho(\dd\theta) \log\{\frac{\dd\rho}{\dd \pi}(\theta)\}$ if $\rho<<\pi$, $\infty$ otherwise, and 
denote $ \mathcal{M}_+^1$ the set of probability distributions $\rho(\dd\theta)$.

\begin{lemma}
\label{lemma-pacbayes}
Assume that {\bf MA}$(1,C)$ holds with $C\geq 1$.
For any fixed $\gamma$ with $0<\gamma\leq (n-1)/(8C)$,
for any $\varepsilon>0$, with probability at least $1-\varepsilon$
on the drawing of the data $\data$,
$$
\int R(\theta) \post\dd\theta - \overline{R}
\leq 2
\inf_{\rho \in \mathcal{M}_+^1} \left\{
\int R(\theta) \rho({\rm d}\theta)- \overline{R}
+ 2\frac{\mathcal{K}(\rho,\pi) + \log\left(
 \frac{4}{\varepsilon}\right)}{\gamma}
\right\}. 
$$	
\end{lemma}

\begin{lemma}
\label{lemma-pacbayes2}
Assume $\textbf{MA}(\infty,C)$ with $C\geq 1$. For any fixed $\gamma$ with $0< \gamma
\leq(n-1)/8$, for any $\epsilon>0$ with probability $1-\epsilon$ on the drawing of $\data$,
\[
\int R(\theta)\post\dd\theta-\bar{R} \leq \inf_{\rho \in \mathcal{M}_+^1}
\left\{\int R(\theta)\rho(d\theta)-\bar{R}+2\frac{\mathcal{K}(\rho,\pi)
+\log\frac2\epsilon}{\gamma}\right\}+\frac{16\gamma}{n-1}.
\]
\end{lemma}

Both lemmas bound the expected risk excess, for a random estimator of $\theta$ generated from $\post$. 

\subsection{Independent Gaussian Prior}\label{sub:gauss_prior}

We now specialise these results to the prior density 
$\pi_\xi(\theta)=\prod_{i=1}^d \varphi(\theta_i;0,\vartheta)$, 
i.e. a product of independent Gaussian distributions $N(0,\vartheta)$;
$\xi=\vartheta$ in this case. 

\begin{thm}
\label{thm-one-ind}
Assume $\mathbf{MA}(1,C)$, $C\geq 1$, $\mathbf{Dens}(c)$, $c>0$, and take 
$\vartheta=\frac2d (1+\frac1{n^2d})$, $\gamma=(n-1)/8C$, then there exists a constant $\alpha=\alpha(c,C,d)$ such that for any $\epsilon>0$, with probability $1-\epsilon$,
\[
\int R(\theta)\pi_\gamma(\theta\vert \data)\dd\theta-\bar{R}\leq 2\inf_{\theta_0}\left\{R(\theta_0)-\bar{R}\right\}
+\alpha \frac{d\log(n)+\log\frac4\epsilon}{n-1}
.\]
\end{thm}

\begin{thm}
\label{thm-inf-ind}
Assume $\mathbf{MA}(\infty,C)$, $C\geq 1$, $\mathbf{Dens}(c)$ $c>0$, and take 
$\vartheta=\frac2d (1+\frac1{n^2d})$, $\gamma=C\sqrt{d n\log(n)}$, there exists a constant $\alpha=\alpha(c,C,d)$ such that for any $\epsilon>0$, with probability $1-\epsilon$,
\[
\int R(\theta)\pi_\gamma(\theta\vert \data)\dd\theta-\bar{R}\leq \inf_{\theta_0}\left\{R(\theta_0)-\bar{R}\right\}
+\alpha \frac{\sqrt{d\log(n)} +\log\frac2\epsilon}{\sqrt{n}}
.\]
\end{thm}

The proof of these results is provided in the appendix. 
It is known that, under $\mathbf{MA}(\kappa,C)$, the rate $(d/n)^{\frac{\kappa}{2\kappa-1}}$
is minimax-optimal for classification problems, see~\cite{Lecue}. Following~\cite{Robbiano2013}
we conjecturate that this rate is also optimal for ranking problems.

\subsection{Spike and slab prior for feature selection}\label{sub:spike_prior}

The independent Gaussian prior considered in the previous section is a natural
choice, but it does not accommodate sparsity, that is, the possibility that only a small subset of the components of $X_i$ actually determine the membership to either class. 
For sparse scenarios, one may use the spike and slab prior of \cite{mitchell1988bayesian}, \cite{GeorgeMcCulloch}, 
$$ \pi_\xi(\theta) = \prod_{i=1}^{d} \left[p \varphi(\theta_i;0,v_1)
    + (1-p) \varphi(\theta_i;0,v_0) \right] $$
with $\xi=(p,v_0,v_1)\in[0,1]\times (\R^+)^2$, and $v_0 \ll v_1$, 
for which we obtain  the following result.
Note $\|\theta\|_0$ is 
the number of non-zero coordinates for $\theta \in\mathbb{R}^d$. 

\begin{thm}\label{thm:spikeslab}
Assume {\bf MA}$(1,C)$ holds with $C\geq 1$, 
{\bf Dens}$(c)$ holds with $c>0$, and take $p=1-\exp(-1/d)$, $v_0 \leq 1/(2nd\log(d))$,
and $\gamma=(n-1)/(8C)$. Then
there is a constant $\alpha=\alpha(C,v_1,c)$ such that
for any $\varepsilon>0$, with probability at least $1-\varepsilon$
on the drawing of the data $\data$,
$$
\int R(\theta) \pi_{\gamma}({\rm d}\theta|\data) - \overline{R}
\leq 2
\inf_{\theta_0} \Biggl\{
 R(\theta_0) - \overline{R} 
+ \alpha \frac{\|\theta_0\|_0 \log(nd) + \log\left(\frac{4}{\varepsilon}\right)}{2(n-1)}
\Biggr\}.
$$
\end{thm}

Compared to Theorem \ref{thm-one-ind}, the bound above increases logarithmically
rather than linearly in $d$, and depends explicitly on $\|\theta\|_0$, 
the sparsity of $\theta$.
This suggests that the spike and slab prior should lead
to better performance  than the Gaussian prior in sparse scenarios.
The rate $\|\theta\|_0 \log(d)/n$ is the same as the one
obtained in sparse regression, see e.g.~\cite{BVDG}.

Finally, note that if $v_0\rightarrow 0$, we recover the more standard prior 
which assigns a point mass at zero for every component. However this leads
to a pseudo-posterior which is a mixture of $2^d$ components that mix Dirac 
masses and continuous distributions, and thus which is more difficult to approximate
(although see the related remark in Section \ref{sub:ep-spike} for Expectation-Propagation). 

\section{Practical implementation of the PAC-Bayesian approach}\label{sec:practical}

\subsection{Choice of hyper-parameters}

Theorems \ref{thm-one-ind}, \ref{thm-inf-ind}, and \ref{thm:spikeslab}
propose specific values for hyper-parameters $\gamma$ and $\xi$, but these values 
depend on some unknown constant $C$. Two data-driven ways to choose $\gamma$ and $\xi$
are (i) cross-validation (which we will use for $\gamma$), and (ii) (pseudo-)evidence
maximisation (which we will use for $\xi$). 

The latter may be justified from intermediate results of our proofs in the appendix, which provide an empirical bound on the expected risk:
\[
\int R(\theta)\post\dd\theta-\bar{R} 
\leq \Psi_{\gamma,n}\inf_{\rho \in \mathcal{M}_+^1}\left(\int R_n(\theta)\rho(d\theta)-\bar{R}_n+\frac{\mathcal{K}(\rho,\pi)+\log\frac2\epsilon}{\gamma}\right)
\]
with $\Psi_{\gamma,n}\leq 2$. The right-hand side is minimised at
$\rho(\dd\theta)=\post\dd\theta$, and the so-obtained bound is 
 $-\Psi_{\gamma,n}\log(\ev)/\gamma$ plus constants. Minimising the upper bound with respect to hyperparameter $\xi$ is therefore equivalent to maximising $\log\ev$
with respect to $\xi$. This is of course akin to the empirical Bayes approach
that is commonly used in probabilistic machine learning.
Regarding $\gamma$ the minimization is more cumbersome because the dependence with the $\log(2/\epsilon)$ term and $\Psi_{n,\gamma}$, which is why we recommend cross-validation
instead.

It seems noteworthy that, beside \cite{Alquier2013}, very few papers discuss the practical implementation of PAC-Bayes, beyond some brief mention of MCMC (Markov chain Monte Carlo). However, estimating the normalising constant of a target density simulated with MCMC is notoriously difficult. In addition, even if one 
decides to fix the hyperparameters to some arbitrary value, MCMC may become slow and difficult to calibrate if the dimension of the sampling space becomes large. 
This is particularly true if the target does not (as in our case) have some specific structure that make it possible to implement Gibbs sampling. 
The two next sections discuss two efficient approaches that make it possible to approximate both the pseudo-posterior $\post$ and its normalising constant,
and also to perform cross-validation with little overhead.

\subsection{Sequential Monte Carlo}\label{sub:smc}

Given the particular structure of the pseudo-posterior $\post$, a natural approach to simulate from $\post$ is to use tempering SMC  \citep[Sequential Monte Carlo][]{DelMoral2006} that is, define a certain sequence
$\gamma_0=0<\gamma_1<\ldots <\gamma_T$, start by sampling 
from the prior $\pi_\xi(\theta)$, then applies successive importance sampling steps, 
from $\posttm$ to $\postt$, leading to importance weights proportional to: 
$$ \frac{\postt}{\posttm} \propto \exp\left\{ - (\gamma_t-\gamma_{t-1}) R_n(\theta)
\right\}. $$
When the importance weights become too skewed, one rejuvenates the particles through
a resampling step (draw particles randomly with replacement, with probability proportional to the weights) and a move step (move particles according to a certain MCMC kernel). 

One big advantage of SMC is that it is very easy to make it fully adaptive. For 
the choice of the successive $\gamma_t$, we follow \cite{Jasra2007} in solving
numerically \eqref{eq:ESS} in order to impose that the Effective sample size has 
a fixed value. This ensures that the degeneracy of the weights always remain under a certain threshold.  
For the MCMC kernel, we use a Gaussian random walk Metropolis step, calibrated
on the covariance matrix of the resampled particles. See Algorithm \ref{algo-smc} for a
summary.

\begin{algorithm}
\caption{Tempering SMC}\label{algo-smc}
\begin{description}
\item[Input] $N$ (number of particles),  $\tau\in(0,1)$ (ESS threshold), $\kappa>0$ (random walk tuning parameter)
\item[Init.] Sample $\theta_0^i\sim\prior$ for $i=1$ to $N$, set $t\leftarrow 1$, $\gamma_0=0$, $Z_0=1$. 
\item[Loop]  
\begin{description}
\item[a.] Solve in $\gamma_t$ the equation
\begin{equation}\label{eq:ESS}
\frac{\{\sum_{i=1}^N w_t(\theta_{t-1}^i)\}^2}
{\sum_{i=1}^N \{w_t(\theta_{t-1}^i))^2\} } = \tau N, 
\quad w_t(\theta) = \exp[ -(\gamma_t-\gamma_{t-1}) R_n(\theta) ]
\end{equation}
using bisection search. If $\gamma_t\geq\gamma_T$, set 
$Z_T=Z_{t-1}\times \left\{\frac{1}{N} \sum_{i=1}^N w_t(\theta_{t-1}^i)\right\}$, and stop. 
\item[b.] Resample: for $i=1$ to $N$, draw $A_t^i$ in $1,\ldots,N$ so that 
$\P(A_t^i=j) = w_t(\theta_{t-1}^j)/\sum_{k=1}^N w_t(\theta_{t-1}^k)$; see Algorithm 1 
in the appendix. 
\item[c.] Sample $\theta_t^i \sim M_t(\theta_{t-1}^{A_t^i},\dd \theta)$ for $i=1$ to $N$
 where $M_t$ is a MCMC kernel that leaves invariant $\pi_t$; see Algorithm 3 in the appendix
 for an instance of such a MCMC kernel, which takes as an input $S=\kappa\hat{\Sigma}$, where
 $\hat{\Sigma}$ is the covariance matrix of the $\theta_{t-1}^{A_t^i}$. 
\item[d.] Set $Z_t = Z_{t-1}\times \left\{ \frac{1}{N} \sum_{i=1}^N w_t(\theta_{t-1}^i) \right\}$.
\end{description}
\end{description}
\end{algorithm}

In our context, tempering SMC brings two extra advantages: it makes it possible
to obtain samples from $\post$ for a whole range of values of $\gamma$, rather than 
a single value. And it provides an approximation of $\ev$ for the same range of $\gamma$ values, 
through the quantity $Z_t$ defined in Algorithm \ref{algo-smc}. 

\subsection{Expectation-Propagation (Gaussian prior)}\label{sub:ep}

The SMC sampler outlined in the previous section works fairly well, and we will use it as gold
standard in our simulations. However, as any other Monte Carlo method, it may be too slow for large datasets.
We now turn our attention to EP \citep[Expectation-Propagation][]{Minka2001}, a general framework to derive fast
approximations to target distributions (and their normalising constants). 

First note that the pseudo-posterior may be rewritten as: 
$$ \post = \frac{1}{\ev} \prior \times \prod_{i,j} f_{ij}(\theta),
\quad f_{ij}(\theta) =  
\exp\left[-\gamma' \ind\{  \left<\theta,X_i-X_j\right><0 \} \right]$$
where $\gamma'=\gamma/\np\nm$, and the product is over all $(i,j)$ such that $Y_i=1$, $Y_j=-1$. EP generates
an approximation of this target distribution based on the same factorisation: 
$$ q(\theta) \propto q_0(\theta) \prod_{i,j} q_{ij}(\theta),
\quad q_{ij}(\theta) = \exp\{- \frac{1}{2} \theta^T Q_{ij} \theta + r_{ij}^T \theta \}.
$$
We consider in the section the case where the prior is Gaussian, as in Section \ref{sub:gauss_prior}. Then
one may set $q_0(\theta)=\pi_\xi(\theta)$. 
The approximating factors are un-normalised Gaussian densities (under a natural parametrisation), 
leading to an overall approximation that is also Gaussian, but other types of exponential family parametrisations
may be considered; see next section and \cite{Seeger2005}. 
EP updates iteratively each site $q_{ij}$ (that is, it updates
the parameters $Q_{ij}$ and $r_{ij}$), conditional on all the sites, by matching the moments of $q$ 
with those of the hybrid distribution 
$$ h_{ij}(\theta) \propto q(\theta) \frac{f_{ij}(\theta)}{q_{ij}(\theta)}
\propto q_0(\theta) f_{ij}(\theta) \prod_{(k,l)\neq (i,j)} f_{kl}(\theta)
$$
where again the product is over all $(k,l)$ such that $Y_k=1$, $Y_l=-1$, and $(k,l)\neq (i,j)$. 

We refer to the appendix for a precise algorithmic description of our EP implementation. We highlight the following points.
First, the site update is particularly simple in our case: 
$$ h_{ij}(\theta) \propto \exp\{\theta^Tr^h_{ij}-\frac12 \theta^T Q^h_{ij}\theta\} \exp\left[ -\gamma' \ind\{ \left< \theta,X_i-X_j\right> < 0 \} \right],
$$
with $\, r^h_{ij} = \sum_{(k,l)\neq (i,j)} r_{kl}$,
$Q^h_{ij} = \sum_{(k,l)\neq (i,j)} Q_{kl}$, 
which may be interpreted as: $\theta$ conditional on $T(\theta)=\left<\theta,X_i-X_j\right>$ has a $d-1$-dimensional
Gaussian distribution, and the distribution of $T(\theta)$ is that of a one-dimensional Gaussian penalised
by a step function. The two first moments of this particular hybrid may therefore be computed exactly, 
and in $\OO(d^2)$ time, as explained in the appendix. The updates can be performed efficiently using the fact that the linear combination $(X_i-X_j)\theta$ is a one dimensional Gaussian. For our numerical experiment we used a parallel version of EP \cite{VanGerven2010}.
The complexity  of our EP
implementation is $\OO(\np\nm d^2+d^3)$. 

Second, EP offers at no extra cost an approximation
of the normalising constant $\ev$ of the target $\post$; in fact, one may even obtain derivatives of this approximated
quantity with respect to hyper-parameters. See again the appendix for more details. 

Third, in the EP framework, cross-validation may be interpreted as dropping all
the factors $q_{ij}$ that depend on a given data-point $X_i$ in the global approximation $q$. This makes it possible to implement cross-validation at little extra cost \citep{Opper2000}. 

\subsection{Expectation-Propagation (spike and slab prior)}\label{sub:ep-spike}

To adapt our EP algorithm to the spike and slab prior of Section \ref{sub:spike_prior}, we
introduce latent variables $Z_k=0/1$ which "choose" for each component $\theta_k$ whether it
comes from a slab, or from a spike, and we consider the joint target
$$ \postsl \propto \left\{\prod_{k=1}^d 
\B(z_k;p)\N(\theta_k;0,v_{z_k}) \right\}
\exp\left[ -\frac{\gamma}{\np\nm}\sum_{ij} \ind\{\left< \theta,X_i-X_j\right> >0 \} \right].
 $$

On top of the $\np\nm$ Gaussian sites defined in the previous section, we add a product of $d$ sites to approximate the prior. Following 
\cite{Hernandez-Lobato2013}, we use 
$$q_{k}(\theta_k,z_k) = \exp\left\{ z_k \log\left(\frac{p_k}{1-p_k}\right) -\frac{1}{2} \theta^2_k u_k +v_{k}\theta_k \right\} $$
that is a (un-normalised) product of an independent Bernoulli distribution for $z_k$, times a Gaussian distribution
for $\theta_k$. Again that the site update is fairly straightforward, and may be implemented
in $\OO(d^2)$ time.  See the appendix for more details. 
Another advantage of this formulation is that we  obtain a Bernoulli approximation of the marginal pseudo-posterior $\postarg{\gamma}{z_i=1}$ to use in feature selection. 
Interestingly taking $v_0$ to be exactly zero also yield stable results corresponding to the case where the spike is a Dirac mass. 

\section{Extension to non-linear scores}\label{sec:nonparam}

To extend our methodology to non-linear score functions, we consider 
the pseudo-posterior 
\begin{equation*}
\postGP 
\propto \priorGP \exp\left\{ -\frac{\gamma}{\np\nm} \sum_{i\in\datap,\,j\in\datam} \ind\{s(X_i)-s(X_j)>0\} \right\} \\
\end{equation*}
where $\priorGP$ is some prior probability measure with respect to an infinite-dimensional functional class. 
Let $s_i=s(X_i)$, $s_{1:n}=(s_1,\ldots,s_n)\in\R^n$, and assume that $\priorGP$ is a GP (Gaussian process)
associated to some kernel $k_\xi(x,x')$, 
then using a standard trick in the GP literature \citep{Rasmussen2006}, one may derive the marginal (posterior)
density (with respect to the $n$-dimensional Lebesgue measure) of $s_{1:n}$ as 
\begin{align*}
\postGPm & \propto \priorGPm \exp\left\{ -\frac{\gamma}{\np\nm} \sum_{i\in\datap,\,j\in\datam} \ind\{s_i-s_j>0\} \right\} \\
\end{align*} 
where $\priorGPm$ denotes the probability density of the $\N(0,K_\xi)$ distribution, and $K_\xi$ is the $n\times n$ matrix 
$\left(k_\xi(X_i,X_j)\right)_{i,j=1}^n$. 

This marginal pseudo-posterior retains essentially the structure of the pseudo-posterior $\post$ for linear scores, except that
the ``parameter'' $s_{1:n}$ is now of dimension $n$. We can apply straightforwardly the SMC sampler of Section \ref{sub:smc}, 
and the EP algorithm of \ref{sub:ep}, to this new target distribution. In fact, for the EP implementation, the particular simple structure
of a single site: 
$$ \exp\left[ -\gamma' \ind\{ s_i-s_j>0 \} \right]
$$
makes it possible to implement a site update in $\OO(1)$ time, leading to an overall complexity $\OO(\np\nm+n^3)$ for the EP algorithm. 

Theoretical results for this approach could be obtained by applying lemmas from e.g. \cite{Vaart2009}, but we leave this for future study.

\section{Numerical Illustration}\label{sec:numerics}

Figure 1 compares the EP approximation with the output of our  SMC sampler, 
on the well-known Pima Indians dataset and a Gaussian prior. 
Marginal first and second order moments essentially match; see the appendix for further details. The subsequent results are obtained with EP.

\begin{figure}[h]
\label{fig:Pimamarg}
\begin{center}
\begin{tabular}{lll}
\subfloat[$\theta_1$]{\includegraphics[scale=0.19]{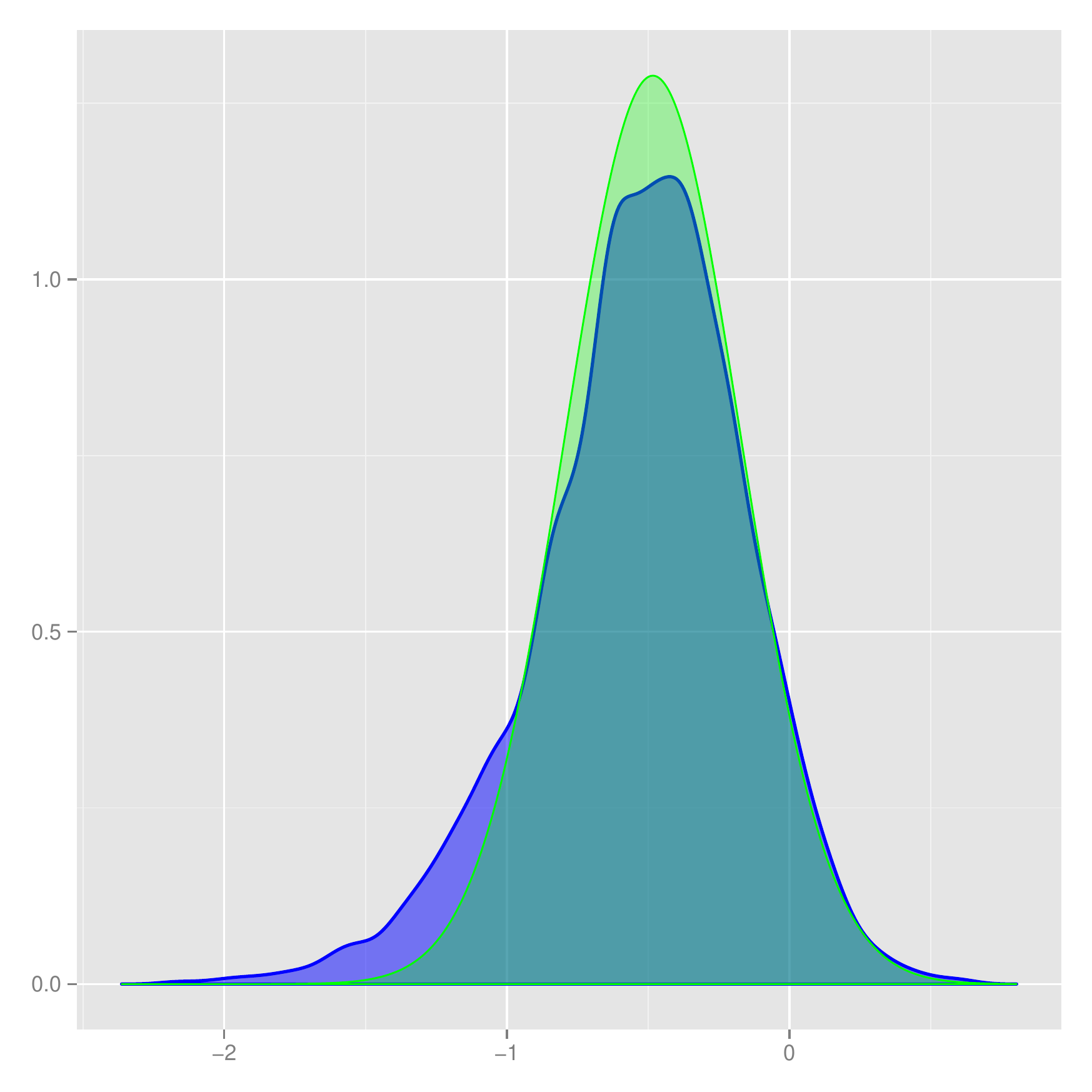}}
& 
\subfloat[$\theta_2$]{\includegraphics[scale=0.19]{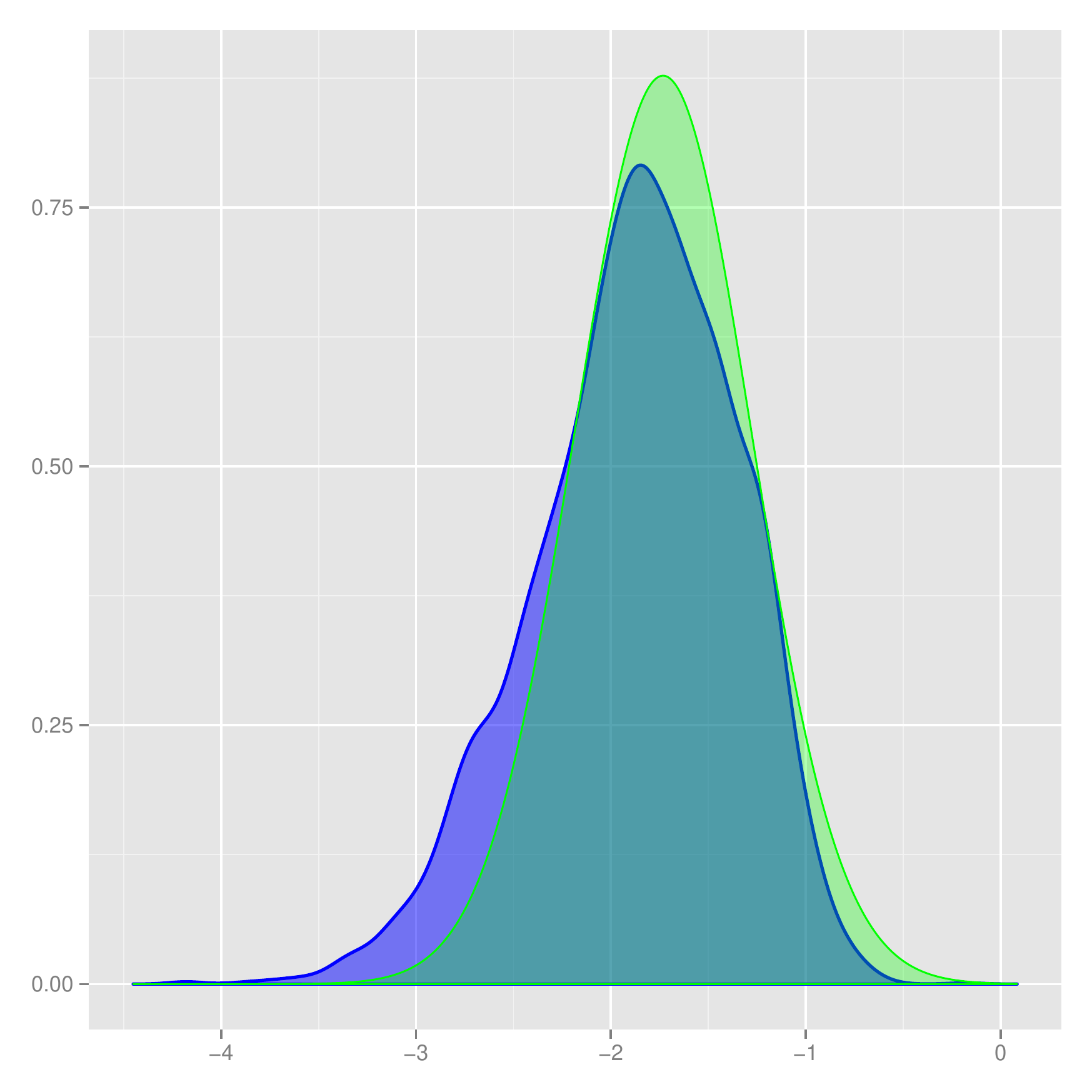} } 
&
\subfloat[$\theta_3$]{\includegraphics[scale=0.19]{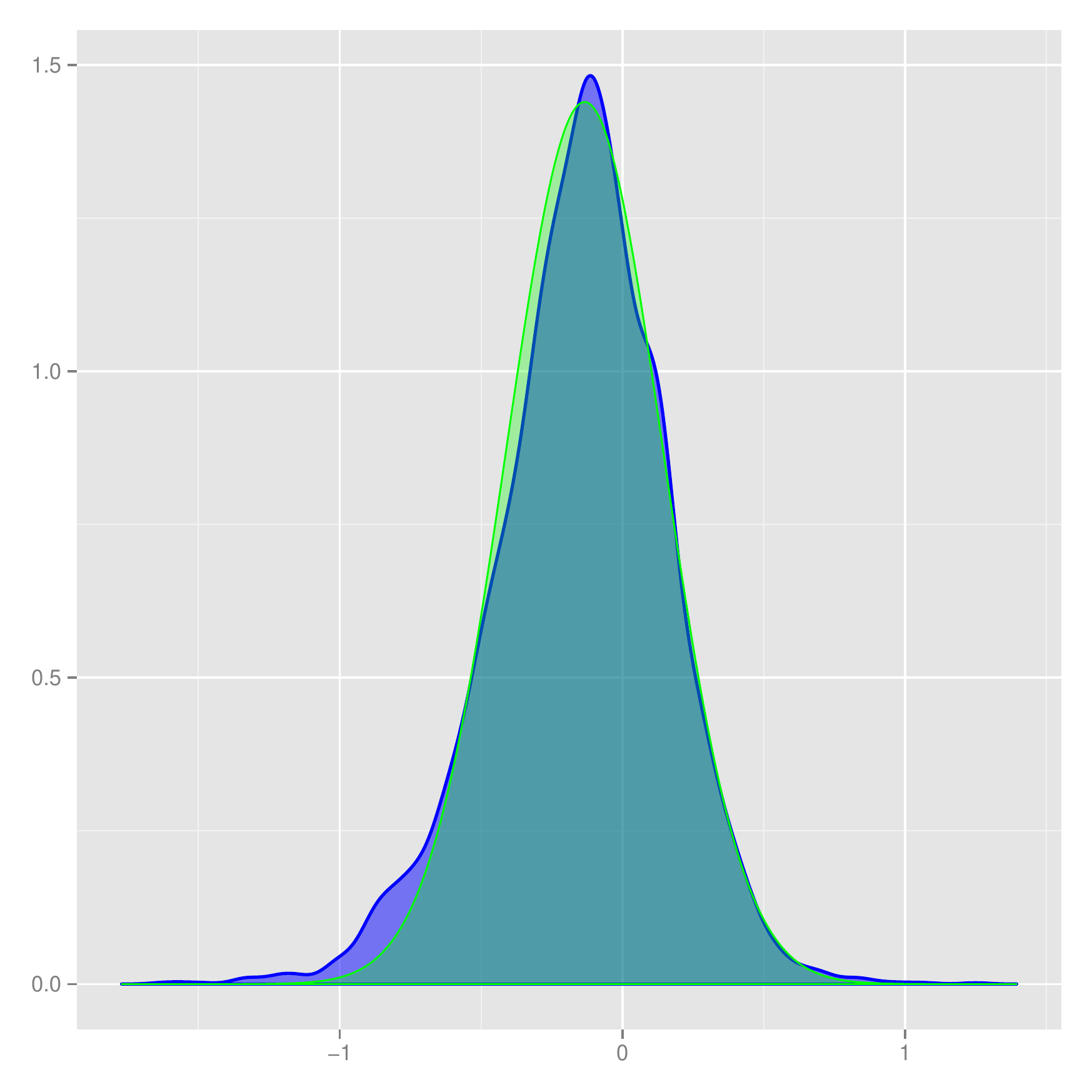} } 
\end{tabular}
\end{center}
\caption{EP Approximation (green), compared to SMC (blue) of the marginal posterior  of the first three coefficients, for Pima dataset (see the appendix for additional analysis).}
\end{figure}


We now compare our PAC-Bayesian approach (computed with EP) with Bayesian logistic
regression (to deal with non-identifiable cases), and with the rankboost algorithm \citep{Freund2003} on different datasets\footnote{All available at http://archive.ics.uci.edu/ml/ }; note that \cite{Cortes2003} showed that the function optimised by rankbook is AUC.

As mentioned in Section \ref{sec:practical}, we set the prior hyperparameters by maximizing the evidence, and we use cross-validation to choose $\gamma$. 
To ensure convergence of EP, when dealing with difficult sites, we use damping \citep{Seeger2005}. The GP version of the algorithm is based on a squared exponential kernel. Table \ref{sample-table} summarises the results; balance refers
to the size of the smaller class in the data (recall that the AUC criterion is 
particularly relevant for unbalanced classification tasks), EP-AUC (resp. GPEP-AUC) refers to the EP approximation of the pseudo-posterior based on our Gaussian prior
(resp. Gaussian process prior). See also Figure 2 for ROC curve
comparisons, and Table 2 in the appendix for a CPU time comparison.

\begin{table}[h]
\begin{center}
\begin{tabular}{lcccccc}
\multicolumn{1}{c}{\bf Dataset}& \multicolumn{1}{c}{\bf Covariates}& \multicolumn{1}{c}{\bf Balance}&\multicolumn{1}{c}{\bf EP-AUC} & \multicolumn{1}{c}{\bf GPEP-AUC} & \multicolumn{1}{c}{\bf Logit} & \multicolumn{1}{c}{\bf Rankboost}
\\ \hline \\
Pima  &7  & 34\% &  0.8617 & 0.8557 & \bf 0.8646 & 0.8224\\
Credit & 60& 28\% & \bf 0.7952 & 0.7922 & 0.7561 & 0.788 \\
DNA  &180 &  22\%    & \bf 0.9814& 0.9812 & 0.9696 & \bf 0.9814 \\
SPECTF  &22 &     50\%   & 0.8684   & 0.8545 &\bf 0.8715 & 0.8684\\
Colon & 2000 & 40\% & 0.7034 & \bf 0.75 & 0.73 &0.5935\\
Glass & 10 & 1\% & \bf 0.9843 & 0.9629 & 0.9029 & 0.9436
\end{tabular}
\caption{Comparison of AUC.} \footnotesize{The Glass dataset has originally more than two classes. We compare the ``silicon'' class against all others.}
\label{sample-table}
\end{center}
\end{table}

Note how the GP approach performs better for the colon data, where the number
of covariates (2000) is very large, but the number of observations is only 40. It seems
also that EP gives a better approximation in this case because of the lower dimensionality of the pseudo-posterior (Figure \ref{fig:GP}).

\begin{figure}[h]
\label{fig:AUCcomp}
\begin{center}
\begin{tabular}{lll}
\subfloat[Rankboost vs EP-AUC on Pima]{\includegraphics[scale=0.19]{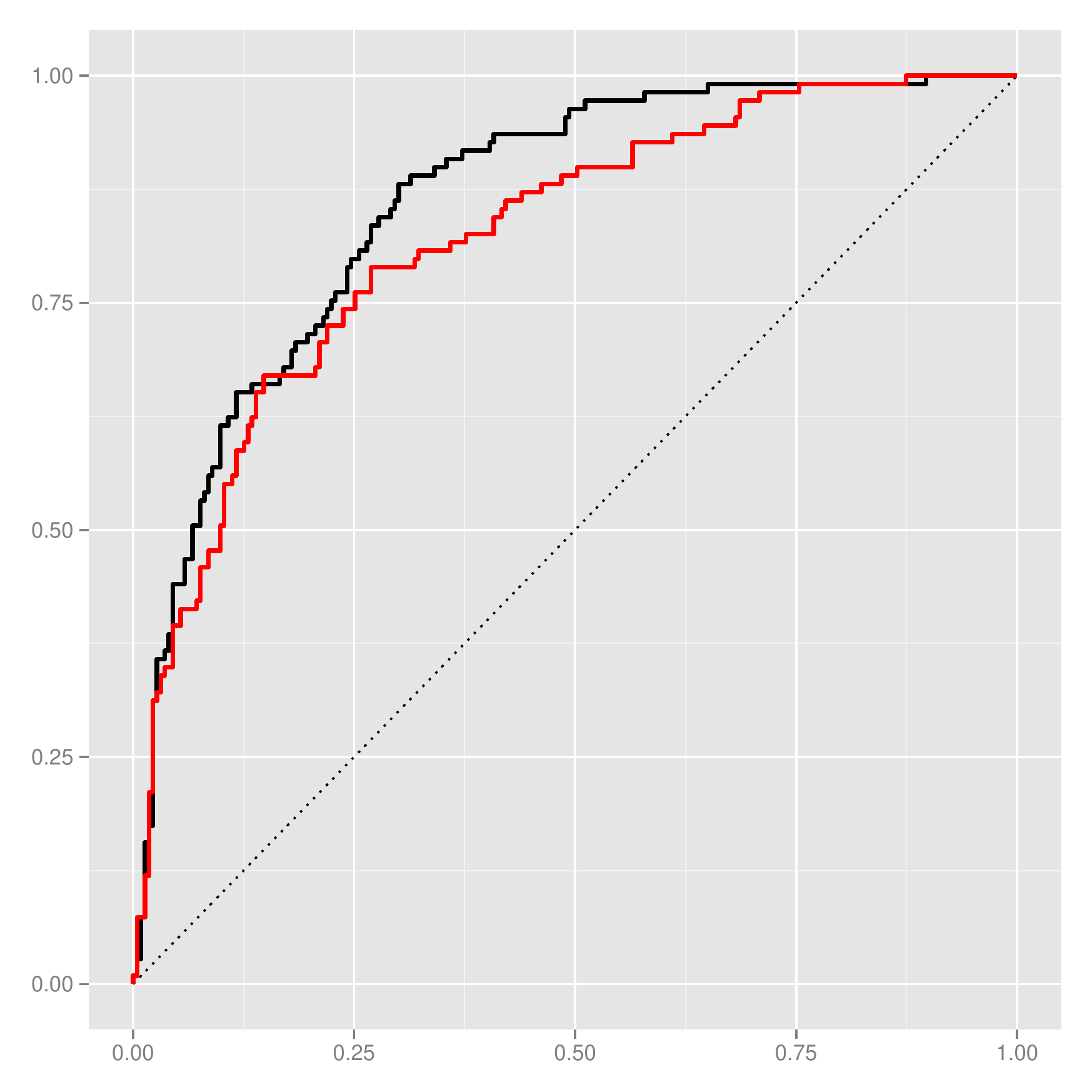}}
& 
\subfloat[Rankboost vs GPEP-AUC on Colon]{\includegraphics[scale=0.19]{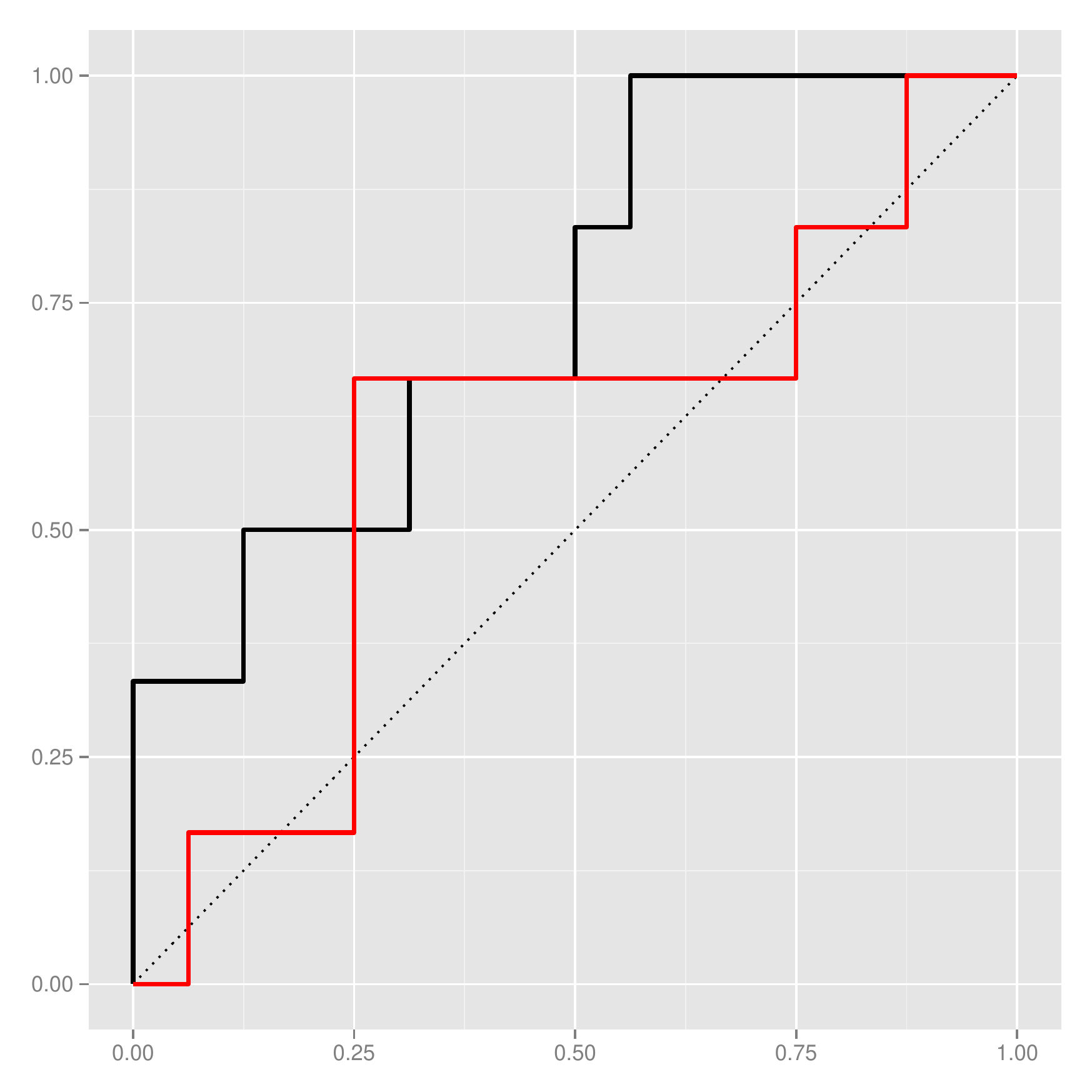}\label{fig:GP}} 
&
\subfloat[Logistic vs EP-AUC on Glass]{\includegraphics[scale=0.19]{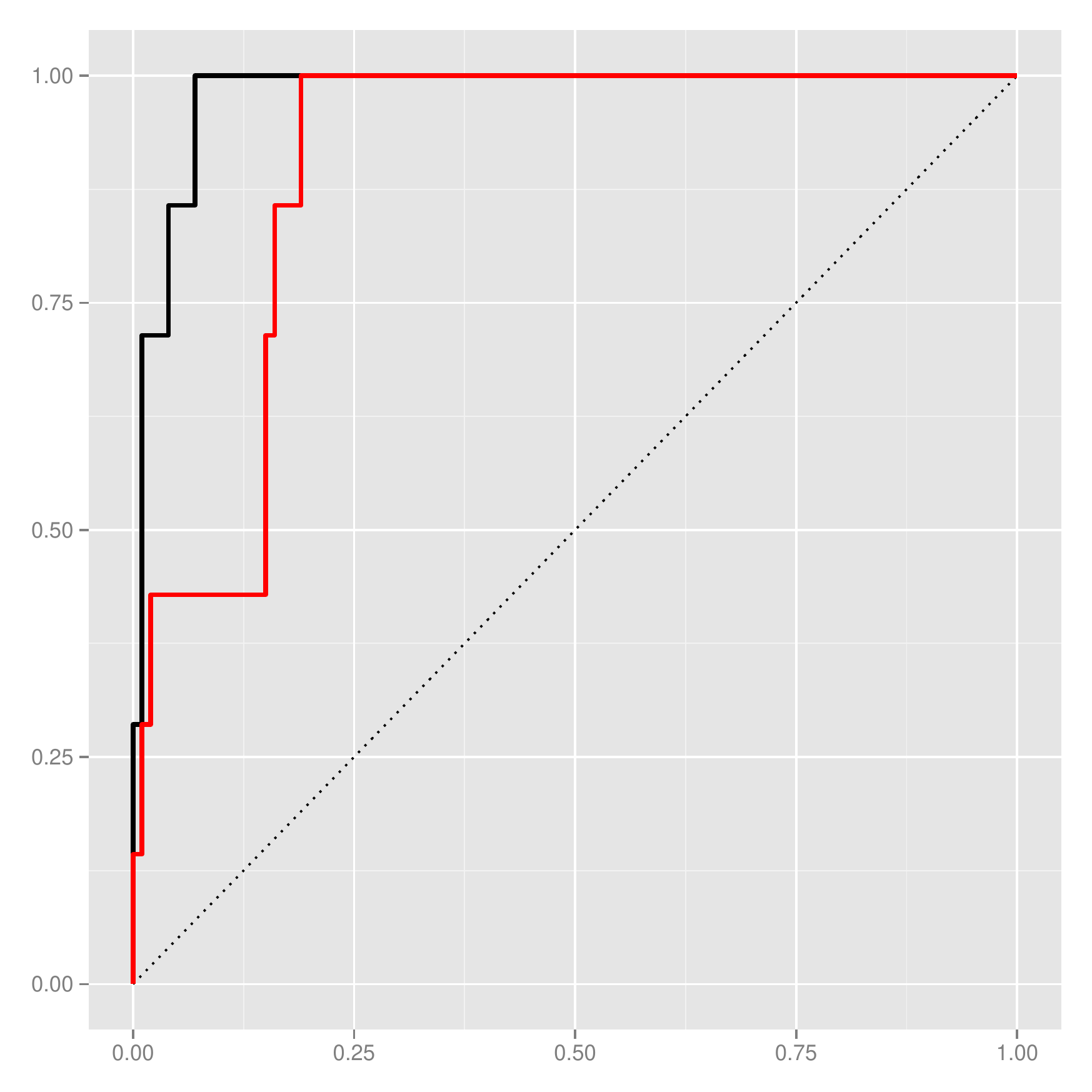}} 
\end{tabular}
\end{center}
\caption{Some ROC curves associated to the example described in a more systematic manner in table \ref{sample-table}. In black is always the PAC version.}
\end{figure}

Finally, we also investigate feature selection for the DNA dataset (180 covariates) using a spike and slab prior. The regularization plot (\ref{fig:regul}) shows how 
certain coefficients shrink to zero as the spike's variance $v_0$ goes to zero, allowing for some sparsity. The aim of a positive variance in the spike  is to absorb negligible effects into it \citep{Rovckova2013}. We observe this effect on figure \ref{fig:regul} where one of the covariates becomes positive when $v_0$ decreases. 
\begin{figure}[H]
\begin{center}
\begin{tabular}{ll}
\subfloat[Regularization plot]{\includegraphics[scale=0.19]{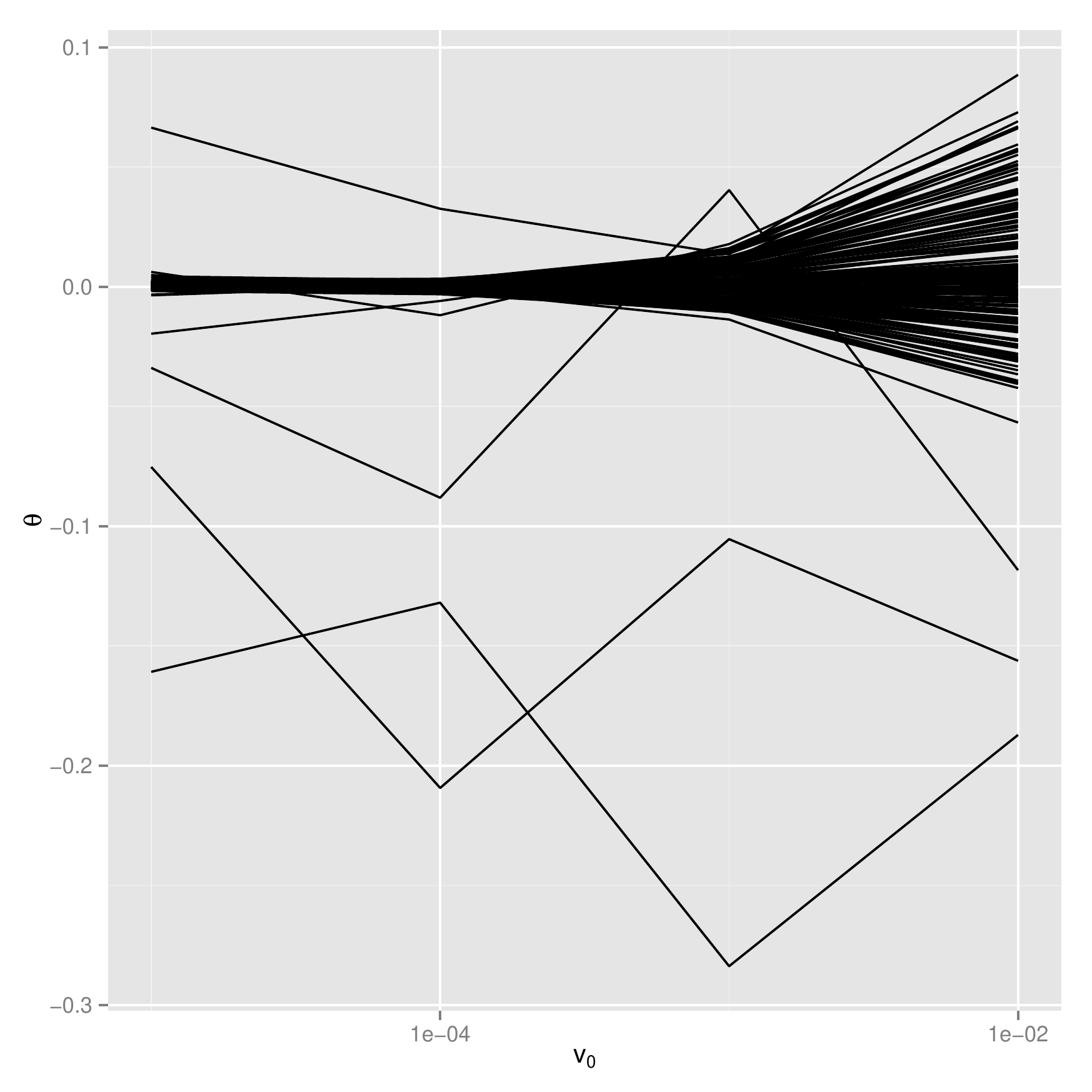}\label{fig:regul}}
& 
\subfloat[Estimate]{\includegraphics[scale=0.19]{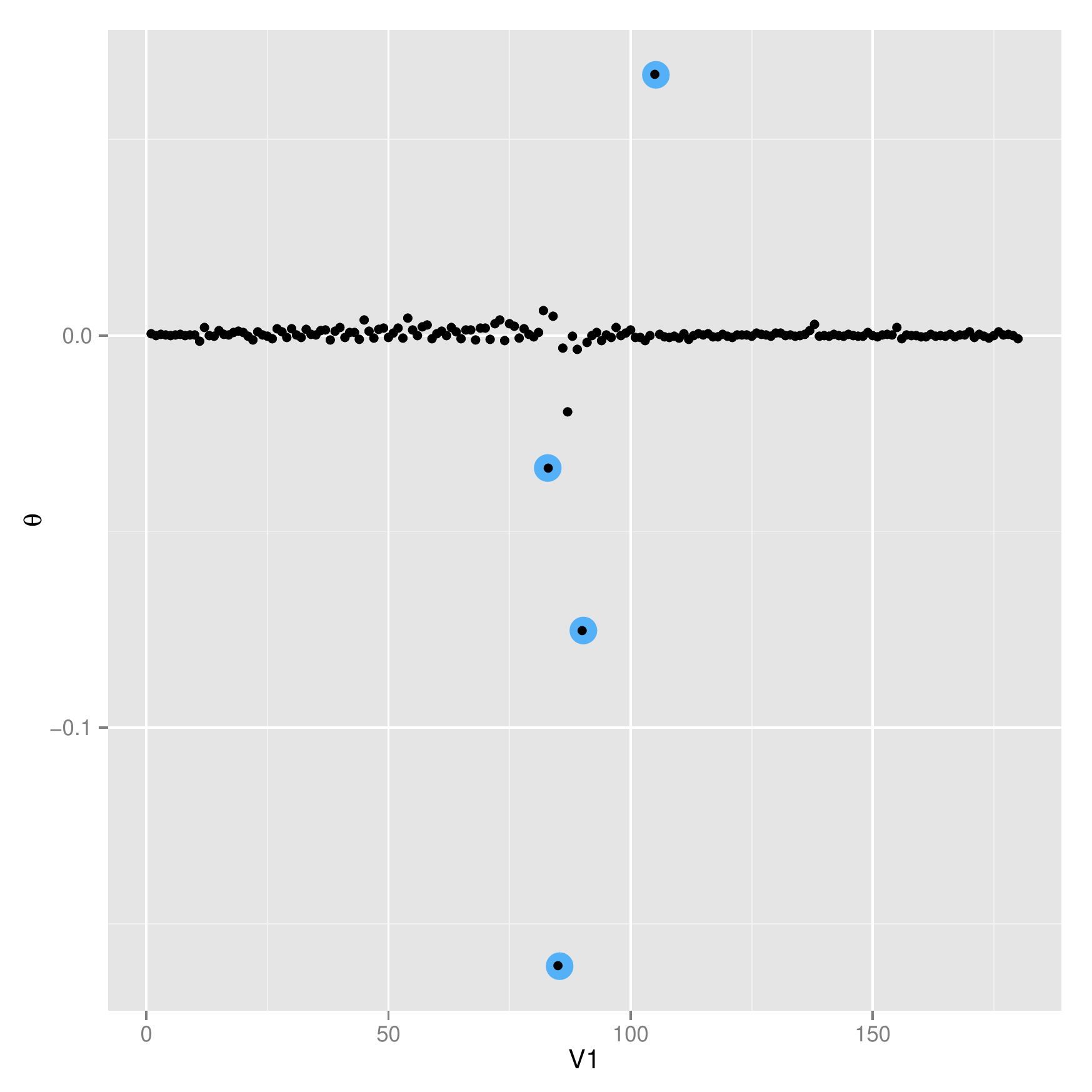} } 
\end{tabular}
\label{fig:FeatureS}
\end{center}
\caption{Regularization plot for $v_0\in\left[10^{-6},0.1\right]$ and estimation for $v_0=10^{-6}$ for DNA dataset; blue circles denote posterior probabilities $\geq0.5$.}
\end{figure}

\section{Conclusion}

The combination of the PAC-Bayesian theory and Expectation-Propagation 
leads to fast and efficient AUC classification algorithms, as observed on a variety of datasets, 
some of them very unbalanced. 
Future work may include extending our approach to more general ranking problems (e.g. multi-class), 
establishing non-asymptotic bounds in the nonparametric case, and reducing
the CPU time by considering only a subset of all the pairs of datapoints. 

\bibliography{biball}

\begin{thebibliography}{29}
\providecommand{\natexlab}[1]{#1}
\providecommand{\url}[1]{\texttt{#1}}
\expandafter\ifx\csname urlstyle\endcsname\relax
  \providecommand{\doi}[1]{doi: #1}\else
  \providecommand{\doi}{doi: \begingroup \urlstyle{rm}\Url}\fi

\bibitem[Alquier(2008)]{Alquier2008}
P.~Alquier.
\newblock Pac-bayesian bounds for randomized empirical risk minimizers.
\newblock \emph{Mathematical Methods of Statistics}, 17\penalty0 (4):\penalty0
  279--304, 2008.

\bibitem[Alquier and Biau(2013)]{Alquier2013}
P.~Alquier and G.~Biau.
\newblock Sparse single-index model.
\newblock \emph{J. Mach. Learn. Res.}, 14\penalty0 (1):\penalty0 243--280,
  2013.

\bibitem[Bishop(2006)]{Bishop2006}
C.~M. Bishop.
\newblock \emph{Pattern Recognition and Machine Learning}, chapter~10.
\newblock Springer, 2006.

\bibitem[B\"uhlmann and van~de Geer(2011)]{BVDG}
P.~B\"uhlmann and S.~van~de Geer.
\newblock \emph{Statistics for High-Dimensionnal Data}.
\newblock Springer, 2011.

\bibitem[Catoni(2007)]{Catoni2007}
O.~Catoni.
\newblock \emph{{PAC-Bayesian Supervised Classification}}, volume~56.
\newblock {IMS Lecture Notes \& Monograph Series}, 2007.

\bibitem[Cl\'emen\c{c}on et~al.(2008{\natexlab{a}})Cl\'emen\c{c}on, Lugosi, and
  Vayatis]{Clemencon2008a}
S.~Cl\'emen\c{c}on, G.~Lugosi, and N.~Vayatis.
\newblock Ranking and empirical minimization of {U}-statistics.
\newblock \emph{Ann. Stat.}, 36\penalty0 (2):\penalty0 844--874, 04
  2008{\natexlab{a}}.

\bibitem[Cl\'emen\c{c}on et~al.(2008{\natexlab{b}})Cl\'emen\c{c}on, Tran, and
  De~Arazoza]{Clemencon2008}
S.~Cl\'emen\c{c}on, V.C. Tran, and H.~De~Arazoza.
\newblock {A stochastic SIR model with contact-tracincing: large population
  limits and statistical inference}.
\newblock \emph{Journal of Biological Dynamics}, 2\penalty0 (4):\penalty0
  392--414, 2008{\natexlab{b}}.

\bibitem[Cortes and Mohri(2003)]{Cortes2003}
C.~Cortes and M.~Mohri.
\newblock Auc optimization vs. error rate minimization.
\newblock In \emph{NIPS}, volume~9, 2003.

\bibitem[Del~Moral et~al.(2006)Del~Moral, Doucet, and Jasra]{DelMoral2006}
P.~Del~Moral, A.~Doucet, and A.~Jasra.
\newblock {Sequential Monte Carlo samplers}.
\newblock \emph{J. R. Statist. Soc. B}, 68\penalty0 (3):\penalty0 411--436,
  2006.
\newblock ISSN 1467-9868.

\bibitem[Freund et~al.(2003)Freund, Iyer, Schapire, and Singer]{Freund2003}
Y.~Freund, R.~Iyer, R.E Schapire, and Y.~Singer.
\newblock An efficient boosting algorithm for combining preferences.
\newblock \emph{J. Mach. Learn. Res.}, 4:\penalty0 933--969, 2003.

\bibitem[George and McCulloch(1993)]{GeorgeMcCulloch}
E.I. George and R.E. McCulloch.
\newblock Variable selection via {G}ibbs sampling.
\newblock \emph{J. Am. Statist. Assoc.}, 88\penalty0 (423):\penalty0 pp.
  881--889, 1993.

\bibitem[Hernandez-Lobato et~al.(2013)Hernandez-Lobato, Hernandez-Lobato, and
  Dupont]{Hernandez-Lobato2013}
D.~Hernandez-Lobato, J.~Hernandez-Lobato, and P.~Dupont.
\newblock {Generalized Spike-and-Slab Priors for Bayesian Group Feature
  Selection Using Expectation Propagation }.
\newblock \emph{J. Mach. Learn. Res.}, 14:\penalty0 1891--1945, 2013.

\bibitem[Hoeffding(1948)]{Hoeffding1948}
W.~Hoeffding.
\newblock {Probability Inequalities for Sums of Random Variables}.
\newblock \emph{Annals of Mathematical Statistics}, 10:\penalty0 293--325,
  1948.

\bibitem[Jasra et~al.(2007)Jasra, Stephens, and Holmes]{Jasra2007}
A.~Jasra, D.~Stephens, and C.~Holmes.
\newblock {On population-based simulation for static inference}.
\newblock \emph{Statist. Comput.}, 17\penalty0 (3):\penalty0 263--279, 2007.

\bibitem[Lecu\'e(2007)]{Lecue}
G.~Lecu\'e.
\newblock M\'ethodes d'agr\'egation: optimalit\'e et vitesses rapides.
\newblock Ph.D. thesis, Universit\'e Paris 6, 2007.

\bibitem[Mammen and Tsybakov(1999)]{Mammen1999}
E.~Mammen and A.~Tsybakov.
\newblock Smooth discrimination analysis.
\newblock \emph{Ann. Stat.}, 27\penalty0 (6):\penalty0 1808--1829, 12 1999.

\bibitem[Massart(2007)]{Massart2007}
P.~Massart.
\newblock \emph{{Concentration Inequalities and Model Selection}}, volume 1896.
\newblock Springer Lecture Notes in Mathematics, 2007.

\bibitem[McAllester(1998)]{McAllester1998}
D.A McAllester.
\newblock {Some PAC-Bayesian theorems}.
\newblock In \emph{Proceedings of the eleventh annual conference on
  Computational learning theory}, pages 230--234. ACM, 1998.

\bibitem[Minka(2001)]{Minka2001}
T.~Minka.
\newblock {Expectation Propagation for approximate Bayesian inference}.
\newblock In \emph{Proc. 17th Conf. Uncertainty Artificial Intelligence}, UAI
  '01, pages 362--369. Morgan Kaufmann Publishers Inc., 2001.

\bibitem[Mitchell and Beauchamp(1988)]{mitchell1988bayesian}
T.~J Mitchell and J.~Beauchamp.
\newblock Bayesian variable selection in linear regression.
\newblock \emph{J. Am. Statist. Assoc.}, 83\penalty0 (404):\penalty0
  1023--1032, 1988.

\bibitem[Opper and Winther(2000)]{Opper2000}
M.~Opper and O.~Winther.
\newblock {Gaussian Processes for Classification: Mean-field Algorithms}.
\newblock \emph{Neural Computation}, 12\penalty0 (11):\penalty0 2655--2684,
  November 2000.

\bibitem[Rasmussen and Williams(2006)]{Rasmussen2006}
C.~Rasmussen and C.~Williams.
\newblock \emph{Gaussian processes for Machine Learning}.
\newblock {MIT} press, 2006.

\bibitem[Robbiano(2013)]{Robbiano2013}
S.~Robbiano.
\newblock Upper bounds and aggregation in bipartite ranking.
\newblock \emph{Elec. J. of Stat.}, 7:\penalty0 1249--1271, 2013.

\bibitem[Ro{\v{c}}kov{\'a} and George(2013)]{Rovckova2013}
V.~Ro{\v{c}}kov{\'a} and E.~George.
\newblock {EMVS: The EM Approach to Bayesian Variable Selection}.
\newblock \emph{J. Am. Statist. Assoc.}, 2013.

\bibitem[Seeger(2005)]{Seeger2005}
M.~Seeger.
\newblock Expectation propagation for exponential families.
\newblock Technical report, U. of California, 2005.

\bibitem[Shawe-Taylor and Williamson(1997)]{Shawe-Taylor1997}
J.~Shawe-Taylor and R.C. Williamson.
\newblock {A PAC analysis of a Bayesian estimator}.
\newblock In \emph{Proc. conf. Computat. learn. theory}, pages 2--9. ACM, 1997.

\bibitem[van~der Vaart and van Zanten(2009)]{Vaart2009}
A.W. van~der Vaart and J.H. van Zanten.
\newblock {Adaptive Bayesian estimation using a Gaussian random field with
  inverse Gamma bandwidth}.
\newblock \emph{Ann. Stat.}, pages 2655--2675, 2009.

\bibitem[Van~Gerven et~al.(2010)Van~Gerven, Cseke, de~Lange, and
  Heskes]{VanGerven2010}
M.~A.J. Van~Gerven, B.~Cseke, F.~P. de~Lange, and T.~Heskes.
\newblock {Efficient Bayesian multivariate fMRI analysis using a sparsifying
  spatio-temporal prior}.
\newblock \emph{NeuroImage}, 50:\penalty0 150--161, 2010.

\bibitem[Yan et~al.(2003)Yan, Dodier, Mozer, and Wolniewicz]{Yanetal2003}
L.~Yan, R.~Dodier, M.~Mozer, and R.~Wolniewicz.
\newblock Optimizing classifier performance via an approximation to the
  {Wilcoxon-Mann-Whitney} statistic.
\newblock \emph{Proc. 20th Int. Conf. Mach. Learn.}, pages 848--855, 2003.

\end{thebibliography}

\appendix

\section{PAC-Bayes bounds for linear scores}\label{sec:theory}


\subsection{Sufficient condition for \textrm{Dens}$\mathrm{(c)}$ }

A simple sufficient condition for \textrm{Dens}$\mathrm{(c)}$ to hold is that 
$(X_1-X_2)/\|X_1-X_2\|$ admits a probability density with respect to the spherical measure of dimension $d-1$ which is bounded above by $B$. Then
\begin{align*}
\mathbb{P}( \left<X_1 - X_2,\theta\right> \geq 0,
  \left<X_1-X_2,\theta'\right> \leq 0)
   & \leq B
   \frac{\arccos \left(\left<\theta,\theta'\right>\right)}{2\pi} \\
   & \leq \frac{B}{2\pi} \sqrt{5-5\left<\theta,\theta'\right>} \\
   & = \frac{B}{2\pi}\sqrt{\frac{5}{2}} \|\theta-\theta'\|.
\end{align*}

\subsection{Proof of Lemma 2.1}
In order to prove Lemma 2.1 we need the following Bernstein inequality.  
\begin{prop}[Bernstein's inequality for U-statistics]
\label{prop-bernstein}
For any $\gamma>0$, for any $\theta\in\mathbb{R}^d$,
$$ \mathbb{E} \exp[\gamma|R_n(\theta)-\overline{R}_n - R(\theta)+
\overline{R}|] \leq
 2 \exp\left[\frac{
\frac{\gamma^2}{n-1}
\mathbb{E}((q^{\theta}_{1,2})^2)
}{\left(1-\frac{4\gamma}{n-1}\right)}\right]. $$ 
\end{prop}

\noindent {\it Proof of Proposition~\ref{prop-bernstein}}.
Fix $\theta$. Remember that
$$q^{\theta}_{i,j}=\mathbf{1}\{\left<\theta,X_i - X_j\right>(Y_i-Y_j)<0\} 
- \mathbf{1}\{[\sigma(X_i)-\sigma(X_j)](Y_i-Y_j)<0\} 
- R(\theta)+\overline{R}$$ so that
 $$ U_n := R_n(\theta)-\overline{R}_n - R(\theta)+
\overline{R} = \frac{1}{n(n-1)}\sum_{i\neq j} q^{\theta}_{i,j} .$$
First, note that 
$$ \mathbb{E} \exp[\gamma |U_n|]
\leq \mathbb{E} \exp[\gamma U_n] + \mathbb{E} \exp[\gamma (-U_n)].
$$
We will only upper bound the first term in the r.h.s., as the upper bound
for the second term may be obtained exactly in the same way (just replace $q^{\theta}_{i,j}$
by $-q^{\theta}_{i,j}$).
Now, use Hoeffding's decomposition \cite{Hoeffding1948}:
this is the technique used by Hoeffding to prove inequalities on U-statistics.
Hoeffding proved that
$$
U_n = \frac{1}{n!}\sum_{\pi} \frac{1}{\lfloor \frac{n}{2}\rfloor}
\sum_{i=1}^{\lfloor \frac{n}{2}\rfloor} q^{\theta}_{\pi(i),\pi(i+\lfloor \frac{n}{2}\rfloor)}
$$
where the sum is taken over all the permutations $\pi$ of $\{1,\dots,n\}$.
Jensen's inequality leads to
\begin{align*}
\mathbb{E} \exp[\gamma U_n]
& = \mathbb{E} \exp\left[\gamma 
\frac{1}{n!}\sum_{\pi} \frac{1}{\lfloor \frac{n}{2}\rfloor}
\sum_{i=1}^{\lfloor \frac{n}{2}\rfloor} q^{\theta}_{\pi(i),\pi(i+\lfloor \frac{n}{2}\rfloor)}
\right] \\
& \leq \frac{1}{n!}\sum_{\pi} \mathbb{E} \exp\left[
\frac{\gamma}{\lfloor \frac{n}{2}\rfloor}
\sum_{i=1}^{\lfloor \frac{n}{2}\rfloor} q^{\theta}_{\pi(i),\pi(i+\lfloor \frac{n}{2}\rfloor)}
\right].
\end{align*}
We now use, for each of the terms in the sum, Massart's version of Bernstein's
inequality \cite{Massart2007} (ineq. (2.21) in Chapter 2, the assumption
is checked by $q^{\theta}_{\pi(i),\pi(i+\lfloor \frac{n}{2}\rfloor)}\in[-2,2]$ so
$\mathbb{E}((q^{\theta}_{\pi(i),\pi(i+\lfloor \frac{n}{2}\rfloor)})^k) \leq
\mathbb{E}((q^{\theta}_{\pi(i),\pi(i+\lfloor \frac{n}{2}\rfloor)})^2) 2^{k-2}$). We obtain:
$$
 \mathbb{E} \exp\left[
\frac{\gamma}{\lfloor \frac{n}{2}\rfloor}
\sum_{i=1}^{\lfloor \frac{n}{2}\rfloor} q^{\theta}_{\pi(i),\pi(i+\lfloor \frac{n}{2}\rfloor)}
\right]
\leq
\exp\left[\frac{
\mathbb{E}((q^{\theta}_{\pi(1),\pi(1+\lfloor \frac{n}{2}\rfloor)})^2)
\frac{\gamma^2}{\lfloor \frac{n}{2}\rfloor}
}{2\left(1-2\frac{\gamma}{\lfloor \frac{n}{2}\rfloor}\right)}\right].
$$
First, note that we have the inequality $\lfloor \frac{n}{2}\rfloor\geq (n-1)/2$. Then,
remark that as the pairs $(X_i,Y_i)$ are iid, we have
$\mathbb{E}((q^{\theta}_{\pi(1),\pi(1+\lfloor \frac{n}{2}\rfloor)})^2)
= \mathbb{E}((q^{\theta}_{1,2})^2)$ so
we have a simpler inequality
$$
 \mathbb{E} \exp\left[
\frac{\gamma}{\lfloor \frac{n}{2}\rfloor}
\sum_{i=1}^{\lfloor \frac{n}{2}\rfloor} q^{\theta}_{\pi(i),\pi(i+\lfloor \frac{n}{2}\rfloor)}
\right]
\leq
\exp\left[\frac{
\mathbb{E}((q^{\theta}_{1,2})^2)
\frac{\gamma^2}{n-1}
}{\left(1-\frac{4\gamma}{n-1}\right)}\right].
$$
This ends the proof of the proposition. $\square$\\\vspace*{5mm}
The following proposition is also of use in the proof of lemma 2.1.
\begin{prop}
For any measure $\rho\in\mathcal{M}_+^1(\Theta)$ and any measurable function $h:\theta\rightarrow \R$ such that $\int \exp(h(\theta))\pi(\dd\theta)<\infty$, we have
$$
\log\left(\int \exp(h(\theta))\pi(\theta)\right)=\sup_{\rho\in\mathcal{M}_+^1}\left(\int h(\theta)\rho(\dd\theta)-\mathcal{K}(\rho,\pi)\right).
$$
In addition if $h$ is bounded by above on the support of $\pi$ the supremum is reached for the Gibbs distribution,
$$
\rho(\dd\theta)\propto \exp\left(h(\theta)\right)\pi(\dd\theta).
$$
\label{prop:Legendre}
\end{prop}
\begin{proof}
e.g. \cite{Catoni2007}.
\end{proof}

\paragraph{Proof of Lemma 2.1}

From the proof of Proposition \ref{prop-bernstein}, and using the short-hand $q_\theta$ for $q^\theta_{1,2}$, we deduce
\begin{equation}
\label{eq:Bernstein}
\mathbb{E}\left[\exp\lbrace\rho\left(\gamma(R_n(\theta)-\bar{R}_n-R(\theta)+\bar{R})\rbrace+\eta(\theta)\right)\right]
\leq \exp\left(\frac{\gamma^2}{n-1}\frac{\rho\left(\mathbb{E}q_\theta^2\right)}{(1-4\frac{\gamma}{n-1})}+\rho\left(\eta(\theta)\right)\right).
\end{equation}

Using proposition \ref{prop:Legendre}, and the fact that $e^x\geq \ind\{x\geq0\}$ we have that
\begin{align*}
&\mathbb{P}\lbrace\sup_{\rho\in\mathcal{M}^1_+(\Theta)} \rho\left(\gamma(R_n(\theta)-\bar{R}_n-R(\theta)+\bar{R})-\eta(\theta)\right)-\mathcal{K}(\rho,\pi)\geq 0\rbrace\\
&\leq\mathbb{E}\left(\pi\lbrace\exp\lbrace\rho\left(\gamma(R_n(\theta)-\bar{R}_n-R(\theta)+\bar{R})-\eta(\theta)\rbrace\right)\right)\rbrace\\
&=\pi\left(\mathbb{E}\lbrace\exp\lbrace\rho\left(\gamma(R_n(\theta)-\bar{R}_n-R(\theta)+\bar{R})-\eta(\theta)\rbrace\right)\right)\rbrace\quad\text{, by Fubini} \\
&\leq \pi\left\{\exp\left(\frac{\gamma^2 \rho(\mathbb{E}q_\theta^2)}{(n-1)(1-\frac{4\gamma}{n-1})}-\rho(\eta(\theta))\right)\right\}\quad\text{, using \eqref{eq:Bernstein}.}
\end{align*}
In the following we take $\eta(\theta)=\log\frac1\epsilon +\frac{\gamma^2}{n-1}
\frac{\rho(\mathbb{E}q_\theta^2)}{(1-4\frac\gamma{n-1})}$ leading to the following result with probability at least $1-\epsilon$, $\forall \rho \in \mathcal{M}_+^1(\Theta)$:
\begin{equation}
\rho(R_n(\theta))-\bar{R}_n\leq\rho(R(\theta))-\bar{R}
+\frac{\mathcal{K}(\rho,\pi)+\log\frac1\epsilon}{\gamma}+
\frac{\gamma}{n-1}\frac{\rho(\mathbb{E}q_\theta^2)}{(1-4\frac\gamma{n-1})}.
\label{Empbound}
\end{equation}
Under $\textbf{MA}(1,C)$ we can write:
\[
\rho(R_n(\theta))-\bar{R}_n\leq\left(1+\frac{\gamma C}{n-1} \frac1{(1-\frac4{ n-1})}\right)\left(\rho(R(\theta))-\bar{R}\right)+\frac{\mathcal{K}(\rho,\pi)+\log\frac1\epsilon}{\gamma}.
\]
Using Bernstein's inequality in the symmetric case, with probability $1-\epsilon$ we can assert that:
\[
 \left(1-\frac{\gamma C}{n-1} \frac1{(1-\gamma\frac4{ n-1})}\right)\left(\rho(R(\theta))-\bar{R}\right)
 \leq\rho(R_n(\theta))-\bar{R}_n+\frac{\mathcal{K}(\rho,\pi)+\log\frac1\epsilon}{\gamma}.
\]
The latter is true in particular for $\rho=\pi(\theta\vert \mathcal{S})$, the Gibbs posterior:
\[
 \left(1-\frac{\gamma C}{n-1} \frac1{(1-\gamma\frac4{ n-1})}\right)\left(\int_{\Theta} R(\theta)\pi_\gamma(d\theta\vert\data)-\bar{R}\right)
 \leq\inf_{\rho\in\mathcal{M}_+^1} \left\{\rho(R_n(\theta))-\bar{R}_n+\frac{\mathcal{K}(\rho,\pi)+\log\frac1\epsilon}{\gamma} \right\}.
\]

Making use of equation \eqref{Empbound} and the fact that $\gamma\leq(n-1)/8C$
we have with probability $1-2\epsilon$:
\[
\left(\int_{\Theta} R_n(\theta)\pi_\gamma(d\theta\vert\data)-\bar{R}_n\right)
\leq 2\inf_{\rho\in\mathcal{M}_+^1}\left(\rho(R(\theta))-\bar{R}+2\frac{\mathcal{K}(\rho,\pi)+\log\frac1\epsilon}{\gamma}\right).\qquad \square
\]

Lemma 2.1 gives some approximately correct finite sample bound under hypothesis {\bf MA}$(1,C)$. It is easy to extend those results to the more general case of {\bf MA}$(\infty,C)$. Note in particular that this assumption is always satisfied for  $C=4$. 

\paragraph{Proof of Lemma 2.2}

First consider in our case that, the margin assumption is always true for $C=4$, $\mathbb{E}(q^2_\theta)\leq 4$, the rest of the proof is similar to that of lemma 2.1. From equation \eqref{Empbound} with the above hypothesis:
\[
\rho(R_n(\theta))-\bar{R}_n\leq\rho(R(\theta))-\bar{R}
+\frac{\mathcal{K}(\rho,\pi)+\log\frac1\epsilon}{\gamma}+\frac{4\gamma}{n-1} \frac1{(1-\frac4{ n-1})}
\]

From the Bernstein inequality with in the symmetric case we get with probability $1-\epsilon$:

\[
\rho(R(\theta))-\bar{R}\leq\rho(R_n(\theta))-\bar{R}_n
+\frac{\mathcal{K}(\rho,\pi)+\log\frac1\epsilon}{\gamma}+\frac{4\gamma}{n-1} \frac1{(1-\frac4{ n-1})}
\]

We get,  after noting that the Gibbs posterior can be written as an infimum (Legendre transform), with probability $1-2\epsilon$:
\[
\int(R(\theta) \pi_\gamma(d\theta\vert \data)-\bar{R}
\leq\inf_{\rho\in \mathcal{M}_+^1(\Theta)}\rho(R(\theta))-\bar{R}
+\frac{\mathcal{K}(\rho,\pi)+\log\frac1\epsilon}{\gamma}
+\frac{16\gamma}{n-1}
\]
(we also used $\gamma\leq (n-1)/8$).

$\square$\linebreak[4]

The two above lemma depend on some class complexity $\mathcal{K}(\rho,\pi)$. The latter can be specialized to different choice of prior measure $\pi$. In the following we propose two specifications to a Gaussian prior and a spike and slab prior.

\subsection{Proof of Theorem 2.3 (Independent Gaussian prior)}

For any $\theta_0 \in\mathbb{R}^{p}$ with $\|\theta_0\|=1$ and $\delta>0$ we put
$$ \rho_{\theta_0,\delta}({\rm d}\theta) \propto
 \mathbf{1}_{\|\theta-\theta_0\|\leq \delta} \pi({\rm d}\theta) .$$
Then we have, from Lemma 2.1, with probability at least
$1-\varepsilon$,
\begin{equation*}
\int R(\theta) \pi_{\gamma}({\rm d}\theta|\data) - \overline{R}
\leq 2
\inf_{\theta_0,\delta} \left\{
\int R(\theta) \rho_{\theta_0,\delta}({\rm d}\theta)- \overline{R}
+ 16C \frac{\mathcal{K}(\rho_{\theta_0,\delta},\pi) + \log\left(
 \frac{4}{\varepsilon}\right)}{(n-1)}
\right\}
\end{equation*}
First, note that
\begin{align*}
R(\theta) & = \mathbb{E}\left(\ind\{\left<\theta,X - X'\right>(Y-Y')<0\}\right) \\
 & = \mathbb{E}\left(\ind\{\left<\theta_0,X - X'\right>(Y-Y')<0\}\right)  \\
 & \quad  + \mathbb{E}\left(\ind\{\left<\theta,X - X'\right>(Y-Y')<0\}
         - \ind\{\left<\theta_0,X - X'\right>(Y-Y')<0\} \right) \\
 & \leq \mathbb{R}(\theta_0) + \mathbb{P}({\rm sign}\left<\theta,X - X'\right>(Y-Y')
    \neq {\rm sign} \left<\theta_0,X - X'\right>(Y-Y')) \\
  & = \mathbb{R}(\theta_0) + \mathbb{P}({\rm sign}\left<\theta,X - X'\right>
    \neq {\rm sign} \left<\theta_0,X - X'\right>) \\
  & \leq R(\theta_0) + c\left\| \frac{\theta}{\|\theta\|}-\theta_0\right\| \\
  & \leq R(\theta_0) + 2 c \|\theta-\theta_0\|.
\end{align*}
As a consequence
$ \int R(\theta) \rho_{\theta_0,\delta}({\rm d}\theta)
 \leq R(\theta_0) + 2c \delta .$
 
The next step is to calculate $\mathcal{K}(\rho_{\theta_0,\delta},\pi)$. We have
$$
\mathcal{K}(\rho_{\theta_0,\delta},\pi)
= \log \frac{1}{\pi\left(\left\{\theta:\|\theta-\theta_0\|\leq \delta\right\}\right)}.
$$
Assuming that $\theta_{0,1}>0$ (the proof is exactly symmetric in the other case) 
\begin{align*}
-\mathcal{K}(\rho_{\theta_0,\delta},\pi)&=\log\pi\left(\{ \theta :\sum_{i=1}^d(\theta_i-\theta_{0,i})^2\leq \delta^2\}\right)\\
&\geq d\log\pi\left(\{ \theta :(\theta_1-\theta_{0,1})^2\leq \frac{\delta^2}{d}\}\right)\\
   & \geq d \log \int_{\frac{\theta_{0,1}}{\sqrt{\vartheta}}-\frac{\delta}{\sqrt{\vartheta d}}}^{
         \frac{\theta_{0,1}}{\sqrt{\vartheta}}+\frac{\delta}{\sqrt{\vartheta d}}} \varphi_{(0,1)}(x) {\rm d}x \\
   & \geq d\log\left(\frac{ \delta}{2 \sqrt{\vartheta d}} \varphi\left(\frac{\theta_{0,1}}{\sqrt{\vartheta}}+\frac{\delta}{\sqrt{\vartheta d}}\right)\right) \\
   & \geq d\log\left(\frac{ \delta}{2 \sqrt{\vartheta d}} \varphi\left(\frac{1}{\sqrt{\vartheta}}+\frac{\delta}{\sqrt{\vartheta d}}\right)\right)\\
   & = d\log\left(\frac{ \delta}{2 \sqrt{2 \pi \vartheta d}} \exp\left[-\frac{1}{2}\left(\frac{1}{\sqrt{\vartheta}}+\frac{\delta}{\sqrt{\vartheta d}}\right)^2\right]\right) \\
&\geq d\log\left\{\frac\delta{2\sqrt{2\pi \vartheta d }}\exp\left(-\frac1\vartheta - \frac{\delta^2}{\vartheta d}\right)\right\}
\end{align*}
\begin{align*}
\mathcal{K}(\rho_{\theta_0,\delta},\pi)&\leq -d\log\{\delta\}+\frac d2\log\{8\pi \vartheta d  \}+\frac1\vartheta + \frac{\delta^2}{\vartheta d}
	\end{align*}

And we can  plug the equation above in the result of lemma 2.1 with $\delta=\frac1n$
\begin{align*}
\int R(\theta)\pi_\gamma(\theta\vert \data)-\bar{R}\leq
2\inf_{\theta_0}\left(R(\theta_0)-\bar{R}+2c\frac1n+\frac2\gamma
\left(d\log\{n\}+\frac d2\log\{8\pi \vartheta d  \}+\frac1\vartheta + \frac{\frac1{n^2}}{\vartheta d}+\log\frac4\epsilon \right)\right)
\end{align*}
Any $\gamma=O(n)$ will lead to a convergence result. Taking $\gamma=(n-1)/8C$ and optimizing in $\vartheta$ we obtain a variance of $\vartheta=\frac {2(1+\frac1{n^2d})}d$.

\subsection{Proof of Theorem 2.4 (Independent Gaussian prior)}

As was done for the previous lemmas we can lift the {\bf MA}$(\infty,C)$ and use the lemma 2.2 instead, which gives rise to Theorem 2.4.

Use Lemma 2.2 and the same steps as in the proof of Theorem \ref{thm-one-ind}, optimize
w.r.t. $\gamma$ and $\vartheta$ to get the result.

We show the same kind of result in the following but for spike and slab priors.

\subsection{Proof of Theorem 2.5 (Spike and slab prior for feature selection)}

As for the proof of theorem \ref{thm-one-ind} we start by defining, 
for any $\theta_0 \in\mathbb{R}^{p}$ with $\|\theta_0\|=1$ and $\delta>0$, 
$$ \rho_{\theta_0,\delta}({\rm d}\theta) \propto
 \mathbf{1}_{\|\theta-\theta_0\|\leq \delta} \pi({\rm d}\theta) $$
so that in the end, by a similar argument as previously it remains only to upper bound the following quantity, 
$$
\mathcal{K}(\rho_{\theta_0,\delta},\pi)
= \log \frac{1}{\pi\left(\left\{\theta:\|\theta-\theta_0\|\leq \delta\right\}\right)}.
$$
Let $\pi_0$ denote the probability distribution such that the $\theta_i$ are iid $\mathcal{N}(0,v_0)$.
So:
\begin{align*}
 -\mathcal{K}(\rho_{\theta_0,\delta},\pi)
 & = \log \pi\left(\left\{\theta: \sum_{i=1}^d (\theta_i-\theta_{0,i})^2 \leq \delta^2 \right\}\right) \\
 & \geq \log \pi\left(\left\{\theta: \forall i, (\theta_i-\theta_{0,i})^2 \leq \frac{\delta^2}{d} \right\}\right) \\
 & = \sum_{i:\theta_{0,i}\neq 0}\log \pi\left(\left\{(\theta_i-\theta_{0,i})^2 \leq \frac{\delta^2}{d} \right\}\right) \\
 & \quad \quad   + \log \pi \left(\left\{\forall i \text{ with }\theta_{0,i}=0,
            \theta_i^2 < \frac{\delta^2}{d} \right\}\right) \\
 & \geq \sum_{i:\theta_{0,i}\neq 0}\log \pi\left(\left\{(\theta_i-\theta_{0,i})^2 \leq \frac{\delta^2}{d} \right\}\right) \\
 & \quad \quad   + \log \pi_{0} \left(\left\{\forall i \text{ with }\theta_{0,i}=0,
            \theta_i^2 < \frac{\delta^2}{d} \right\}\right) + d\log(1-p) \\
 & = \sum_{i:\theta_{0,i}\neq 0}\log \pi\left(\left\{(\theta_i-\theta_{0,i})^2 \leq \frac{\delta^2}{d} \right\}\right) \\
 & \quad \quad   + \log \left[1-\pi_0 \left(\left\{\exists i,\theta_{0,i}=0,
            \theta_i^2 > \frac{\delta^2}{d} \right\}\right) \right] + d\log(1-p) \\
 & \geq \sum_{i:\theta_{0,i}\neq 0}\log \pi\left(\left\{(\theta_i-\theta_{0,i})^2 \leq \frac{\delta^2}{d} \right\}\right) \\
 & \quad \quad   +  \log \left[1-\sum_{i:\theta_i=0} \pi_0 \left(\left\{
       \theta_i^2 > \frac{\delta^2}{d} \right\}\right) \right] + d\log(1-p).
\end{align*}
Assume first that $i$ is such that
$\theta_{0,i}= 0$. Then:
\begin{align*}
\pi_0 \left(\left\{ \theta_i^2 > \frac{\delta^2}{d} \right\}\right)
   & = \pi_0 \left(\left\{ \left|\frac{\theta_i}{\sqrt{v_0}}\right| > \frac{\delta}{\sqrt{v_0 d}} \right\}\right) \\
   & \leq \exp\left( - \frac{\delta^2}{2 v_0 d} \right),
\end{align*}
and so
$$
\sum_{i:\theta_{0,i}=0} \pi_0 \left(\left\{
       \theta_i^2 > \frac{\delta^2}{d} \right\}\right) \leq
d \exp\left( - \frac{\delta^2}{2 v_0 d} \right) \leq \frac{1}{2}
$$
as soon as $v_0 \leq \delta^2/(2d \log(d))$. Then, assume that $i$ is such that
$\theta_{0,i}\neq 0$. Now assume that $\theta_{0,i}>0$ (the proof is exactly symmetric if $\theta_{0,i}<0$):
\begin{align*}
 \pi\left(\left\{\theta: (\theta_i-\theta_{0,i})^2 \leq \frac{\delta^2}{d}\right\}\right) 
   & \geq p \int_{\frac{\theta_{0,i}}{\sqrt{v_1}}-\frac{\delta}{\sqrt{v_1 d}}}^{
         \frac{\theta_{0,i}}{\sqrt{v_1}}+\frac{\delta}{\sqrt{v_1 d}}} \varphi_{(0,1)}(x) {\rm d}x \\
   & \geq \frac{p \delta}{2 \sqrt{v_1 d}} \varphi\left(\frac{\theta_{0,i}}{\sqrt{v_1}}+\frac{\delta}{\sqrt{v_1 d}}\right) \\
   & \geq \frac{p \delta}{2 \sqrt{v_1 d}} \varphi\left(\frac{1}{\sqrt{v_1}}+\frac{\delta}{\sqrt{v_1 d}}\right) \\
   & = \frac{p \delta}{2 \sqrt{2 \pi v_1 d}} \exp\left[-\frac{1}{2}\left(\frac{1}{\sqrt{v_1}}+\frac{\delta}{\sqrt{v_1 d}}\right)^2\right] \\
   & \geq \frac{ p \delta}{2 \sqrt{2 \pi v_1 d}} \exp\left[-\frac{1}{v_1}- \frac{\delta^2}{v_1 d}\right].
\end{align*}
Putting everything together:
\begin{align*}
 \mathcal{K}(\rho_{\theta_0,\delta},\pi)
 & \leq - \|\theta_0\|_0 \log\left(\frac{ p \delta}{2 \sqrt{2 \pi v_1 d}}
       \exp\left[-\frac{1}{v_1}- \frac{\delta^2}{v_1 d}\right]\right) + \log(2) + d\log\frac{1}{1-p} \\
 & = \|\theta_0\|_0 \left[ \log\left(\frac{2 \sqrt{2 \pi v_1 d}}{ p \delta}\right) + \frac{1}{v_1}
                 + \frac{\delta^2}{v_1 d} \right]+ \log(2) + d\log\frac{1}{1-p}.
\end{align*}
So, we have:
\begin{multline*}
\int R(\theta) \pi_{\gamma}({\rm d}\theta|\data) - \overline{R}
\leq 2
\inf_{\theta_0,\delta} \Biggl\{
 R(\theta_0) - \overline{R} + 2c\delta
\\
+ 16C \frac{\|\theta_0\|_0 \left[ \log\left(\frac{2 \sqrt{2 \pi v_1 d}}{ p \delta}\right) + \frac{1}{v_1}
                 + \frac{\delta^2}{v_1 d} \right]+ \log(2) + d\log\frac{1}{1-p} + \log\left(
 \frac{4}{\varepsilon}\right)}{(n-1)}
\Biggr\}
\end{multline*}

\section{Practical implementation of the PAC-Bayesian approach}\label{sec:practical}


\subsection{Sequential Monte Carlo}\label{sub:smc}

The resampling scheme we use in our SMC sampler is systematic resampling,
see Algorithm \ref{algo-sys}.

\begin{algorithm}[h]
\caption{Systematic resampling
\label{algo-sys}}
\begin{description}
\item[Input:] Normalised weights 
$W_t^j\eqdef w_t(\theta_{t-1}^j)/\sum_{i=1}^N w_t(\theta_{t-1}^i)$.
\item[Output:] indices $A^i\in\{1,\ldots,N\}$, for $i=1,\ldots,N$. 
\item[a.] Sample $U\sim \Unif$.
\item[b.] Compute cumulative weights as 
$C^n=\sum_{m=1}^n NW^m$.
\item[c.] Set $s\leftarrow U$, $m\leftarrow 1$. 
\item[d.] \textbf{For} $n= 1:N$
\item[]$\qquad$ \textbf{While} $C^m<s$ \textbf{do} $m\leftarrow m+1$. 
\item[]$\qquad$ $A^n\leftarrow m$, and $s\leftarrow s+1$. 
\item[]$\quad$ \textbf{End For}
\end{description}
\end{algorithm}

To move the particles while leaving invariant the current target $\post$, 
we use the standard random walk Metropolis strategy, but scaled to the 
current set of particles, as outlined by Algorithm \ref{algo-RW}. 

\begin{algorithm}[h]
\caption{Gaussian random walk Metropolis step
\label{algo-RW}}
\begin{description}
\item[Input:] $\theta$, $S$ ($d\times d$ positive matrix) 
\item[Output:] $\theta_\mathrm{next}$
\item[a.] Sample $\theta_\mathrm{prop}\sim \N(\theta,S)$.
\item[b.] Sample $U\sim \Unif$.
\item[c.] If $\log(U)\leq \log \postarg{\gamma}{\theta_\mathrm{prop}}/\post$, 
set $\theta_\mathrm{next}\leftarrow\theta_\mathrm{prop}$, otherwise
set $\theta_\mathrm{next}\leftarrow\theta$.
\end{description}
\end{algorithm}

\subsection{Expectation-Propagation (Gaussian prior)}\label{sub:ep}

EP aims at approximating posterior distributions of the form,
\[
\pi(\theta\vert \data)= \frac1{Z_\pi}P_0(\theta)\prod_{i=1}^nt_i(\theta)
\]
by approximating each site $t_i(\theta)$ by a distribution from an exponential family $q_i(\theta)$. The algorithm cycles through each site, computes the cavity distribution $Q^{\backslash i}(\theta)\propto Q(\theta)q_i^{-1}(\theta)$ and minimizes the Kullback-Leibler divergence between $Q^{\backslash i}(\theta)t_i(\theta)$ and the global approximation $Q(\theta)$. This is efficiently done by using properties of the exponential family (e.g. \cite{Bishop2006}). 

In the Gaussian case the EP approximation can be written as a product of some prior and a product of sites:
$$
Q(\theta)\propto \mathcal{N}(\theta;0,\Sigma)\prod_{i,j}q_{ij}(\theta),
$$
for which the sites are unnormalized Gaussians for the natural parametrization $q_{ij}(\theta)\propto\exp\left(-\frac12 \theta^T Q_{ij}\theta + \theta r_{ij} \right)$. We can equivalently use the one dimensional representation $q_{ij}(s_{ij})\propto\exp\left(-\frac12 s_{ij}^2 K_{ij} + s_{ij}h_{ij} \right)$, going from one to the other is easily done by multiplying $\theta$ by $(e_i-e_j)X$ where $\forall i\in\lbrace 1,\cdots,n\rbrace,\quad e_i$ is a vector of zeroes with one on the i-th line. Hence we keep in memory only $(K_{ij})_{ij}$ and $(h_{ij})_{ij}$.

While computing the cavity moment we must compute  $\left( Q-(X_i-X_j)(X_i-X_j)K_{ij}\right)$ and its inverse. The latter can be computed efficiently using Woodbury formula. Equivalently one could use similar tricks where only the Cholesky factorisation is saved and updated as in \cite{Seeger2005}. By precomputing some matrix multiplication the later cavity moment computation can be done in complexity $\OO(p^2)$.

To update the sites we compute normalizing constant $Z_{ij}=\int \mathcal{N}(s;m^{\backslash ij},\sigma^{\backslash ij})t_{ij}(s)\dd s$ and use properties of exponential families.
\begin{algorithm}[h]
\caption{parallel EP for Gaussian Prior
\label{algo-EP}}
\begin{description}
\item[Input:] $\vartheta$, $\gamma$ 
\item[Output:] $m$ and $V$  
\item[Init:] $V\leftarrow\Sigma$, $m\leftarrow 0$ 
\item[] \textbf{Untill} Convergence \textbf{Do}
\item[] \textbf{For} all sites $(i,j)$\textbf{Do} in parallel
\begin{itemize}
\item[a.] Compute the cavity moments $m^{\backslash ij}$, $V^{\backslash ij}$
\item[b.] Compute the 1st and 2nd order moments of $q^{\backslash ij}(s_{ij})t_{ij}(s_ij)$
\item[c.] Update $K_{ij}$ and $h_{ij}$
\end{itemize}
\item[] \textbf{End For}
\item[] Update $V=(\Sigma^{-1}+\sum_{ij}(X_i-X_j)^T(X_i-X_j) K_{ij})^{-1}$, $m=V(\sum_{ij}(X_i-X_j)h_{ij})$
\item[] \textbf{End While}
\end{description}
\end{algorithm}

\paragraph{Normalising Constant}

The normalizing constant of the posterior can be computed using EP. We have that for each sites $t_{ij}(\theta)=C_{ij}q_{ij}(\theta)$ we replace those sites in integral we wish to approximate,
$$
\int \mathcal{N}(\theta;0,\Sigma)\prod_{ij}t_{ij}(\theta)\dd \theta\simeq \prod_{ij} C_{ij} \int \mathcal{N}(\theta;0,\Sigma)\prod_{ij}q_{ij}(\theta)\dd \theta
$$
The integral on the right hand side is a Gaussian convolution and is therefore also Gaussian. The $C_{ij}$s can be approximated by matching the zeroth order moment in the site update. As noted in the paper we can also compute the derivatives with respect to some prior hyper-parameter (see \cite{Seeger2005}).   

\subsection{Expectation-Propagation (spike and slab prior)}

The posterior can be written as 
$$
\pi(\theta | \data)\propto\prod_{i,j}t_{ij}(\theta)\prod_{k=1}^d t_k(\theta_k,z_k)\mathcal{B}er(z_k;p),
$$
where $z_k\in\lbrace 0,1 \rbrace$ codes the origin of $\theta_k$, spike/slab, and where $t_k(\theta_k,z_k)\propto z_k\mathcal{N}(\theta_k;0,v_0)+(1-z_k)\mathcal{N}(\theta_k;0,v_1)$. The approximation given by EP is of the form,
$$
Q(\theta,z)\propto\prod_{i,j}q_{ij}(\theta)\prod_{k=1}^d q_k(\theta_k,z_k)\mathcal{B}er(z_k;p_k),
$$
where $q_k(\theta_k,z_k)\propto \mathcal{B}er(z_k,p_k) \mathcal{N}(\theta_k;m_k,\sigma_k^2)$, and $t_{ij}(\theta)$ is as in the previous section. The cavity moments are easy to compute as the approximation is Gaussian in $\theta$ and Bernoulli in $z$. In both cases we can deduce cavity moments because division is stable inside those classes of functions. We get some distribution $Q^{\backslash k}(\theta_k)\propto \mathcal{B}er(z_k;p^{\backslash k})\mathcal{N}(\theta_k;m^{\backslash k},\sigma^{2,\backslash k})$. We can compute the normalizing constant of the distribution $Q^{\backslash ij}(\theta)t_k(\theta_k,z_k)$, namely,
$$
Z_k=p^{\backslash k}\int \mathcal{N}(\theta_k;0,v_0)\mathcal{N}(\theta_k;m^{\backslash k},\sigma^{2,\backslash k})\dd \theta_k+(1-p^{\backslash k})\int \mathcal{N}(\theta_k;0,v_0)\mathcal{N}(\theta_k;m^{\backslash k},\sigma^{2,\backslash k})\dd \theta_k
$$
Where we can find the update by computing the derivatives of $\log Z_k$ with respect to $p^{\backslash k}$, $m^{\backslash k}$ and $\sigma^{2,\backslash k}$

Initialization for the Gaussian is done to a given $\Sigma_0$ that will be subtracted later on. The initial $p_k$s are taken such that the approximation equals the prior $p$ at the first iteration.


\section{Numerical illustration}

\begin{figure}[h]
\begin{center}
\caption{Comparison of the output of the two algorithms}
\begin{tabular}{lll}
\subfloat[1st covariate]{\includegraphics[scale=0.2]{graph/AUC1.pdf}} &
\subfloat[2nd covariate]{\includegraphics[scale=0.2]{graph/AUC2.pdf} } &
\subfloat[3rd covariate]{\includegraphics[scale=0.2]{graph/AUC3.pdf} } \\
\subfloat[4th covariate]{\includegraphics[scale=0.2]{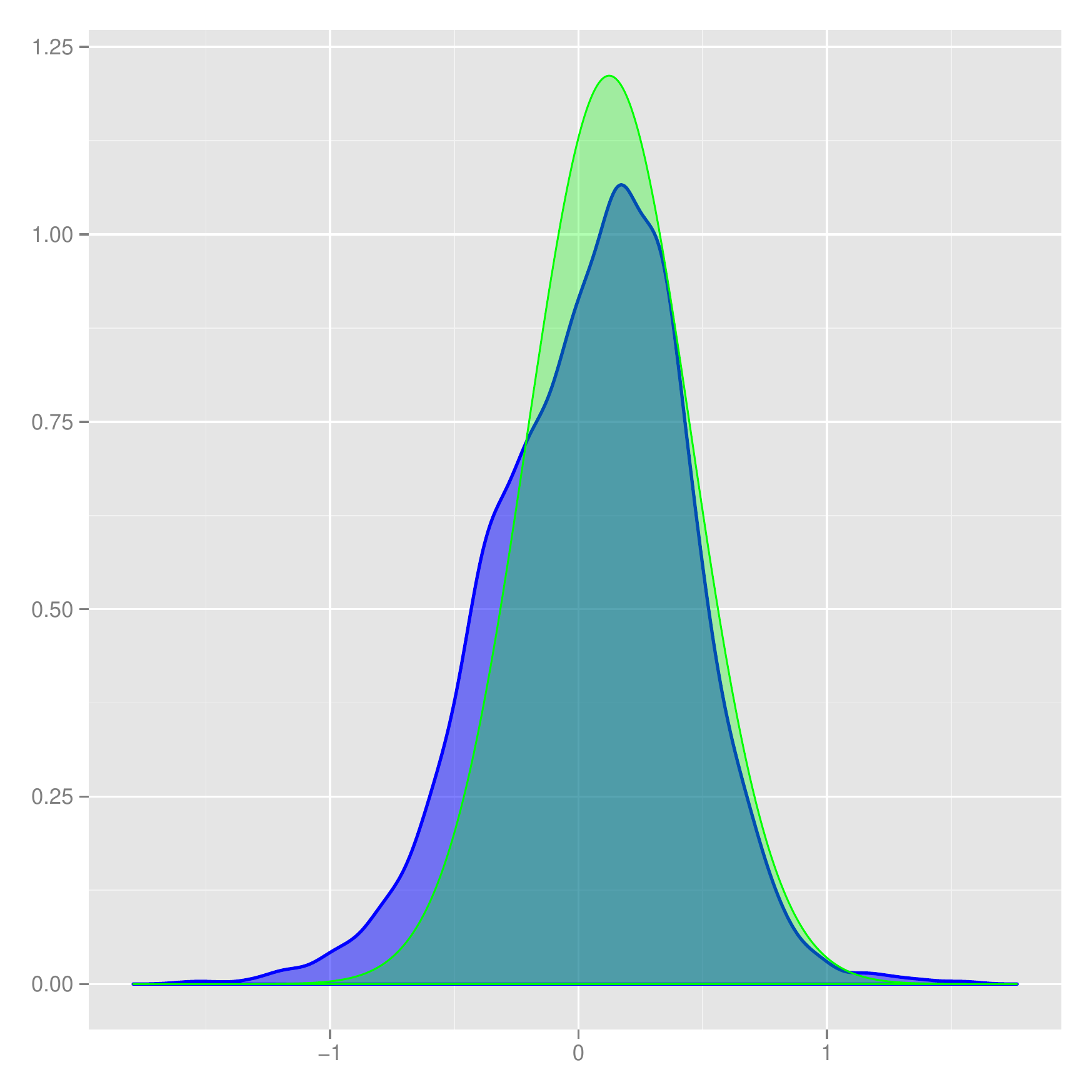} } &
\subfloat[5th covariate]{\includegraphics[scale=0.2]{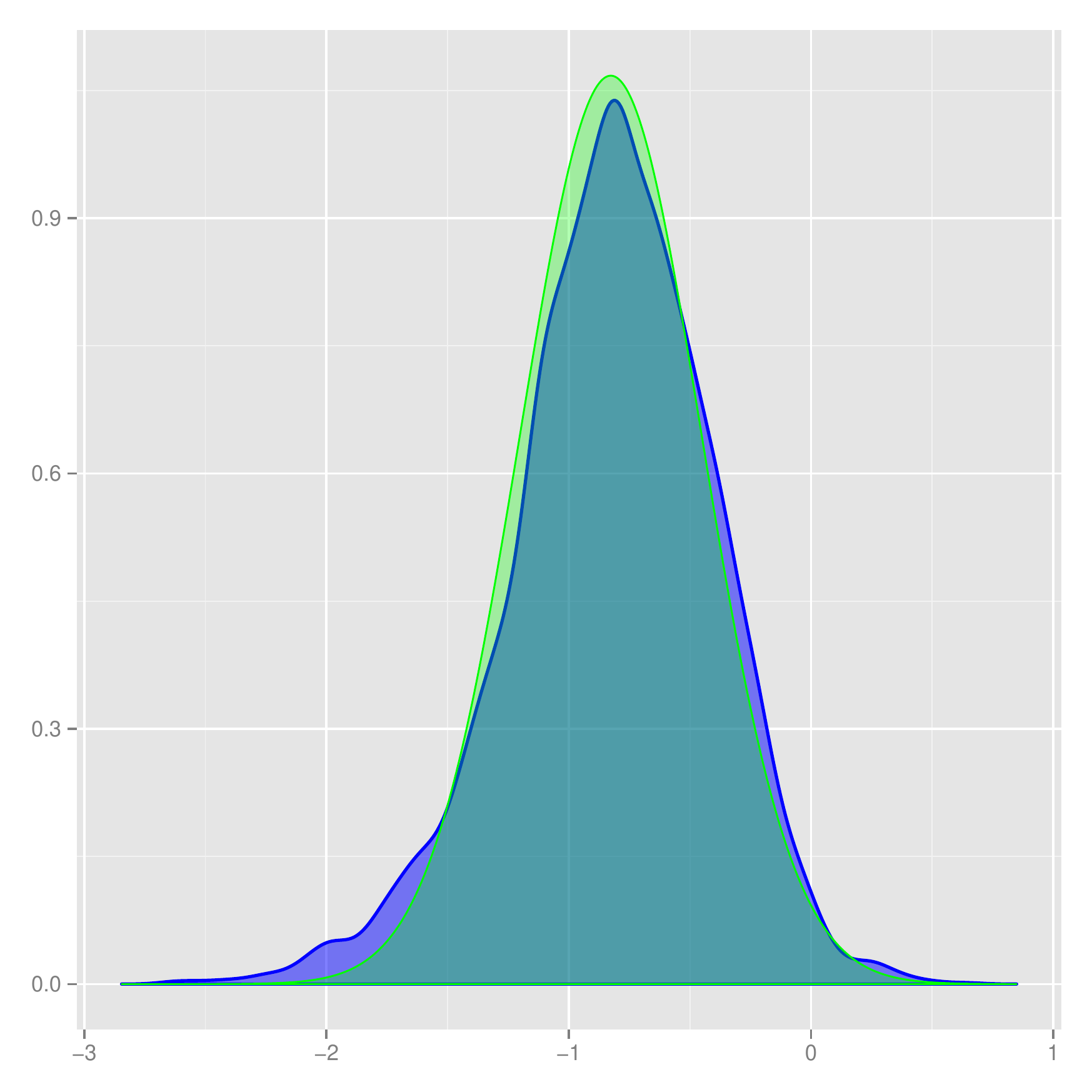} } &
\subfloat[6th covariate]{\includegraphics[scale=0.2]{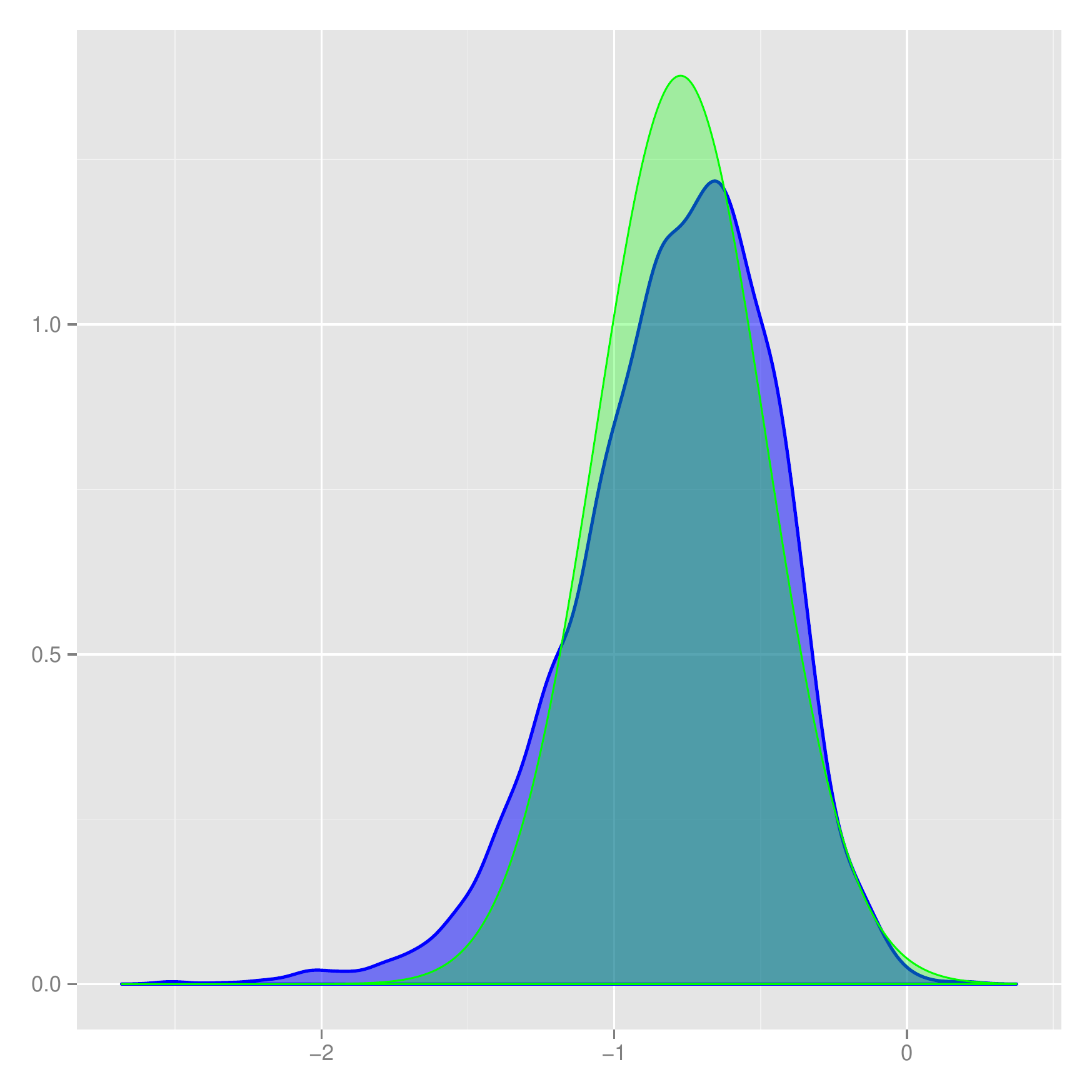} }\\
\subfloat[7th covariate]{\includegraphics[scale=0.2]{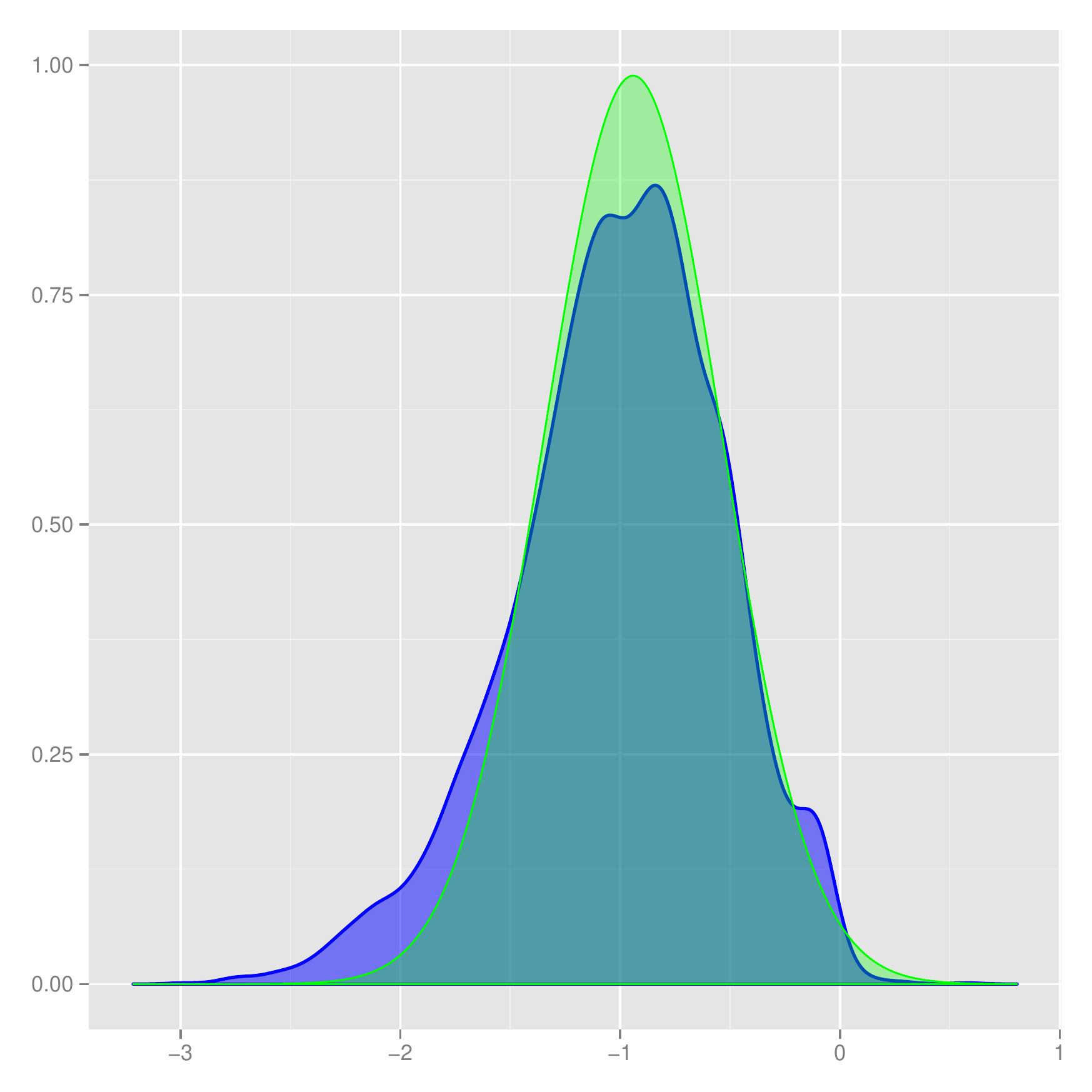} } &
\end{tabular}
\label{fig:density}
\begin{minipage}{15cm}
\footnotesize{Comparison of the Gaussian approximation  obtained by  Fractional EP (green) with the true density generated by  SMC (blue) on the Pima indians dataset}
\vspace*{3mm}
\hrule
\end{minipage}
\end{center}
\end{figure}

Figure \ref{fig:density} shows the posterior marginals as given by EP and tempering SMC. The later is exact in the sense that the only error stems from Monte Carlo; we see that the mode is well approximated however the variance is slightly underestimated. 

In Table \ref{CPU-table} we show the CPU times in seconds, on all dataset studied. Experiments where run with a i7-3720QM CPU @ 2.60GHz intel processor with 6144 KB cache. Our linear model is overall faster on those datasets. A caveat is that Rankboost is implemented in Matlab, while our implementation is in C. 

\begin{table}[h]
	\begin{center}
		\begin{tabular}{lccccc}
			\multicolumn{1}{c}{\bf Dataset}& \multicolumn{1}{c}{\bf Covariates}& \multicolumn{1}{c}{\bf Balance}&\multicolumn{1}{c}{\bf EP-AUC} & \multicolumn{1}{c}{\bf GPEP-AUC} & \multicolumn{1}{c}{\bf Rankboost}
			\\ \hline \\
			Pima  &7  & 34\% &  \bf 0.06 & 7.75  & 3.26\\
			Credit & 60& 28\% & \bf 1.98 & 7.59  & 56.54 \\
			DNA  &180 &  22\%    & \bf 11.26& 63.47 & 141.60 \\
			SPECTF  &22 &     50\%   & \bf 0.25   & 63.47 & 3.55\\
			Colon & 2000 & 40\% & 636.63 & \bf 60.99 & 156.85\\
			Glass & 10 & 1\% & \bf 0.23 & 1.33 & 2.36
		\end{tabular}
	\caption{Computation times in seconds}
	\label{CPU-table}
	\end{center}
\end{table}

\end{document}